\documentclass[12pt]{article}
\usepackage{authblk}
\usepackage{hyperref}
\usepackage{graphicx}
\usepackage{caption}
\usepackage{subfig}
\usepackage{multirow}
\usepackage{float}
\usepackage{amsmath}
\usepackage{amsthm}
\usepackage{amsfonts}
\usepackage{amssymb}
\usepackage[table]{xcolor}
\usepackage{xcolor}
\usepackage{endnotes}
\usepackage{xparse}
\usepackage{algpseudocode}
\usepackage[Algorithm]{algorithm}
\usepackage[margin=2cm]{geometry}
\usepackage{setspace}
\usepackage[squaren]{SIunits}
\usepackage[title]{appendix}
\usepackage[normalem]{ulem}
\usepackage{mathtools}

\usepackage{tikz}
	\usetikzlibrary{shapes,arrows}
	\usetikzlibrary{positioning}
	\usetikzlibrary{calc,shapes, positioning}
	\usetikzlibrary{patterns}
	\usetikzlibrary{decorations.pathreplacing}

\newtheorem{theorem}{Theorem}
\newtheorem{lemma}[theorem]{Lemma}

\newtheorem{definition}[theorem]{Definition}

\def \R {\mathbb{R}}

\def \E {\mathbb{E}}

\newcommand{\abs}[1]{\left | {#1} \right |}
\newcommand{\ip}[2]{\langle {#1},{#2} \rangle}

\usepackage{soul}
\usepackage[normalem]{ulem}

\newcommand{\rvc}[2]{
\protect\textcolor{red}{#1}%
    \IfNoValueF{#2}{\protect\endnote{#2}%
    }%
}

\newcommand\MyBox[2]{
	\fbox{\lower0.75cm
		\vbox to 1.7cm{\vfil
			\hbox to 1.7cm{\hfil\parbox{1.4cm}{#1\\#2}\hfil}
			\vfil}%
	}%
}

\begin{document}

\title{A max-cut approach to heterogeneity in cryo-electron microscopy}

\author
{Yariv Aizenbud${^1}$~~Yoel Shkolnisky${^1}$\\
${^1}$School of Mathematical Sciences, Tel Aviv University, Israel
}

\maketitle
\begin{abstract}
The field of cryo-electron microscopy has made astounding advancements in the past few years, mainly due to advancements in electron detectors' technology. Yet, one of the key open challenges of the field remains the processing of heterogeneous data sets, produced from samples containing particles at several different conformational states.
For such data sets, the algorithms must include some classification procedure to identify homogeneous groups within the data, so that the images in each group correspond to the same underlying structure.
The fundamental importance of the heterogeneity problem in cryo-electron microscopy has drawn many research efforts, and resulted in significant progress in classification algorithms for heterogeneous data sets. While these algorithms are extremely useful and effective in practice, they lack rigourous mathematical analysis and performance guarantees.

In this paper, we attempt to make the first steps towards rigorous mathematical analysis of the heterogeneity problem in cryo-electron microscopy.
To that end, we present an algorithm for processing heterogeneous data sets, and prove accuracy and stability bounds for it. We also suggest an extension of this algorithm that combines the classification and reconstruction steps. We demonstrate it on simulated data, and compare its performance to the state-of-the-art algorithm in RELION.

\end{abstract}


\smallskip


\section{Introduction}
The study of the molecular structure of complex proteins has drawn many efforts in the past few decades.
Among the many structure determination methodologies available, cryo-electron microscopy (cryo-EM) single particle reconstruction (SPR)~\cite{Walz2015} has become the state-of-the-art tool for structure determination~\cite{Methodoftheyear2015}, mainly due to the introduction of direct electron detectors, as well as due improvements in the accompanying data processing algorithms~\cite{Bai2015}. These improvements resulted in three-dimensional molecular models with unprecedented resolutions as high as 2.2~\angstrom~\cite{Bartesaghi2015}. The tremendous importance and impact  of cryo-electron microscopy was unequivocally acknowledged by awarding the 2017 Nobel Prize in Chemistry to the pioneers of the field ``for developing cryo-electron microscopy for the high-resolution structure determination of biomolecules in solution"~\cite{nobel2017}.

The process of resolving the three-dimensional structure of a molecule using cryo-EM SPR typically consists of the following steps. First, a sample consisting of many copies of the investigated molecule is rapidly frozen and is imaged by an electron-microscope. This results in a large image of the sample, known as a micrograph, containing multiple randomly oriented and positioned particle images. The individual particle images are then segmented from this micrograph, resulting in a stack of images, where each image corresponds to a projection of one of the copies of the molecule in the sample. This process is repeated until a sufficiently large stack of raw images is obtained. The images in the stack are then clustered, aligned, and averaged, resulting in images of improved quality, known as class averages, which are used to remove low-quality raw images from the stack. Next, a low-resolution model of the molecule is generated from the class averages or the raw images themselves, and this model is refined using the stack of raw images into a high-resolution  model~\cite{Frank2006,scheres2012relion}.

Algorithms for estimating a low-resolution model of the investigated molecule are often based on detecting common lines between pairs of images~\cite{vanHeel, wang2013orientation, Shkolnisky2012}. The underlying assumption of such algorithms is that all images were generated from exactly the same underlying molecule. Unfortunately, in many cases, it is impossible to purify a sample consisting of only a single type of molecule. In such cases, another step is required, which classifies the images into groups such that all images in the same group correspond to the same type of molecule. This problem is known as the heterogeneity problem.

There are several approaches to address the heterogeneity problem. These approaches follow one of four paradigms: maximum-likelihood estimation (MLE)~\cite{scheres2010maximum, scheres2012relion, scheres2016processing, punjani2017cryosparc, zheng2012three}, covariance matrix estimation~\cite{anden2015covariance, katsevich2013covariance,Liao2015, penczek2006method, anden2017structural}, reduction to semidefinite programming (SDP)~\cite{lederman2016representation}, and graph partitioning~\cite{shatsky2010automated,herman2008classification}. Out of the four, maximum-likelihood approaches are the most widely used and have proven very effective in practice. The idea behind maximum-likelihood approaches is to formulate a function that attains its minimum for the correct assignment of the images into the different classes (and orientations), and to search for this minimum using an optimization method such as expectation-maximization.
The number of unknowns in the resulting optimization problem is typically very high.
In addition to the huge search-space for the optimization, the function that is being minimized is highly non-convex, and it is therefore impossible to provide any guarantees regrading its convergence properties. There are results on convergence of the EM algorithm in some special cases~\cite{wu1983convergence}, but those do not apply in practical settings. That being said, maximum-likelihood methods are currently considered state-of-the-art and are being widely applied to experimental data sets.

The second approach to the heterogeneity problem consists of algorithms that are based on covariance matrix estimation~\cite{anden2015covariance, katsevich2013covariance,Liao2015, penczek2006method}. These algorithms consider the unknown three-dimensional volumes of the molecules as realizations of a random variable, and attempt to approximate the covariance between any two voxels of the random variable. In~\cite{anden2015covariance}, it is described how to use the (estimated) covariance matrix to reconstruct the different volumes.
The main current limitation of these approaches is their computational complexity, as the covariance matrix that corresponds to a volume of size $L \times L \times L$ is of size $L^{3} \times L^{3}$, and thus it is extremely computationally expensive to estimate it for large $ L $. Consequently, covariance-based methods typically require down-sampling the input images to a sufficiently low resolution. The resulting low-resolution volumes may be used to initialize the aforementioned MLE approaches.
Recently, a method that uses MLE to directly approximate the principal components of the covariance matrix has been proposed in~\cite{tagare2015directly}.

A third approach to the heterogeneity problem has been recently proposed in~\cite{lederman2016representation}, based on the ``Non-Unique Games'' framework, which was presented and related to cryo-EM in~\cite{bandeira2015non}. The
Non-Unique Games framework provides a representation theoretic approach
to studying problems of alignment over compact~\cite{bandeira2015non} and non-compact~\cite{ozyesil2018synchronization} groups. It offers convex relaxations for alignment problems, which are formulated as semidefinite programs (SDPs). Under certain circumstances, the framework guarantees global optimality of the solution.
Although this approach seems to be very general, it is currently too computationally expensive and has not yet been applied to cryo-EM data.

Finally, the fourth approach to the heterogeneity problem is based on devising a similarity measure for pairs of images, and then assigning the images into different classes based on that similarity measure~\cite{shatsky2010automated,herman2008classification}. The algorithm we propose goes in this direction, in a way that allows us to analyze its properties and to derive some bounds on its accuracy. The idea behind our algorithm is to use a score function to measure the similarity between every pair of images, to construct a weighted graph whose weights are given by these similarities, and then to use a maximum $K$-cut algorithm to determine which images belong to each class. Our score function assigns low score to images of the same underlying structure, assigns high score to images of different underlying structures, and is designed such that the numerical properties of the resulting algorithm can be explicitly stated and proved.
In particular, unlike other works that use graph-based partitioning, in this paper we present an iterative algorithm that uses not only the information of common lines' correlations but also uses the estimated imaging directions and their consistency with the common lines.

Of particular relevance to our work is the algorithm in~\cite{shatsky2010automated}, named ``multi-model reconstruction", and so we discuss it in more details. The multi-model reconstruction algorithm takes as an input the projection images, as well as some initial ``average model'' and the number $K$ of different conformations present in the data. Then, it generates 98 reference images by projecting the initial model in 98 directions (the number 98 is determined by the chosen discretization of the space of three-dimensional orientations), and assigns each input image to one of the directions (by projection matching to the reference images). This results in a partition of the input images to 98 groups. Next, each of the 98 groups is partitioned into $K$ clusters, where the intuition is to partition the set of images assigned to each reference image according to the $K$ underlying structures. Then each of the $ 98 K $ clusters is averaged to get $ 98 K $ sub-class averages. The next step of the multi-model reconstruction algorithm is to partition the entire set of sub-class averages into $K$ clusters, by trying to group together all images that correspond to the same underlying structure. To that end, a graph is constructed, whose nodes correspond to the sub-class averages, and whose (weighted) edges encode the similarities between the sub-class averages as follows. For each pair of sub-class averages corresponding to different directions (out of the 98), the weight of the corresponding edge is set according to the similarity of the common lines between the sub-class averages (similar images have high weight between them). For each pair of sub-class averages of the same direction, the weight is set to zero (so that two sub-class averages that correspond to the same direction will not be ``connected"). Spectral clustering on this graph results in the assignment of all images to $K$ groups, and each such group is used to generate an initial model of one of the conformations present in the data set. This process generates the $K$ starting points for subsequent iterative refinement.

There are several key differences between the multi-model reconstruction algorithm of~\cite{shatsky2010automated} and the one proposed in the current paper. First, the multi-model reconstruction algorithm requires an initial model for generating reference images, whereas the theoretical guarantees of the algorithm in the current paper are independent of any initial model. Second, the key step in the multi-model reconstruction algorithm of clustering the images assigned to the same reference image into $K$ clusters ignores relations between images assigned to different reference images. In contrast, our algorithm considers all images in the data set at once, without any such local per-view processing. Third, the multi-model reconstruction algorithm estimates the initial models for refinement in one step, while our algorithm iterates to improve its estimates. Last,~\cite{shatsky2010automated} does not present any mathematical analysis to justify the convergence of the algorithm or to demonstrate its properties. This last point is the distinguishing property of our algorithm. It may therefore serve as a first step towards rigorous mathematical analysis of practical algorithms for the heterogeneity problem.

The paper is organized as follows. In Section~\ref{sec:prob_for}, we formulate the problem and set up the required notation. In Section~\ref{sec:preliminaries}, we review the required mathematical background. We describe the algorithm in details in Section~\ref{sec:alg}, and prove its performance bounds in Section~\ref{sec:proofs}. We demonstrate the algorithm on simulated data and compare its performance with RELION~\cite{scheres2016processing} in Section~\ref{sec:experimental_results}. Some concluding remarks are given in Section~\ref{sec:conclusion}.

\section{Problem formulation} \label{sec:prob_for}

We start by presenting the homogeneous setting of cryo-EM single particle reconstruction (SPR). Mathematically, a molecule is modeled as a function $\phi:\mathbb{R}^{3} \to \mathbb{R}$. Each image $P_{i}$ generated by the electron microscope is produced by rotating the volume $ \phi $ by some unknown rotation $ R_i $, projecting the rotated volume along the $z$ direction, and convolving the resulting two-dimensional image with the point spread function (PSF) $H_i$ of the microscope. This results in a clean projection image which is shifted by some unknown shift $ (\Delta x_i,\Delta y_i)  \in \mathbb{R}^2 $ before the noise is added. Explicitly, we can write the forward model in cryo-EM SPR as
\begin{align}\label{eq:projection integral_with_CTF}
\begin{aligned}
P'_{i}(x,y) &= \left(H_i *\int_{-\infty}^{\infty} \phi \left( R_{i}{r}
\right)\, dz \right), \quad r=(x,y,z)^{T},\quad i=1,\ldots,N,
\\
P_{i}(x,y) &= P'_{i}(x+\Delta x_i,y+\Delta y_i) + \mbox{noise}.
\end{aligned}
\end{align}
The point spread function $ H_i $ and it's Fourier transform, called the contrast transfer function (CTF), are known from the experimental settings or can be estimated. 
In this (homogeneous) setting, all images are assumed to correspond to exactly the same underlying molecule $\phi$. The goal of cryo-EM structure determination is to recover $\phi$ given a finite set of its images $P_{1},\ldots,P_{N}$ generated according to~\eqref{eq:projection integral_with_CTF}.

In the discrete heterogeneous setting of cryo-EM structure determination considered here, we have $K$ underlying molecules $\phi_{1},\ldots,\phi_{K}$, and each two-dimensional image $P_i$ is generated according to~\eqref{eq:projection integral_with_CTF} from one of $\phi_{1},\ldots,\phi_{K}$. 
Explicitly, we use the following simplified model of heterogeneous cryo-EM SPR,
\begin{equation}\label{eq:projection integral}
P_{i}(x,y) = \int_{-\infty}^{\infty} \phi_k \left( R_{i}{r} \right
)\, dz, \quad r=(x,y,z)^{T},\quad i=1,\ldots,N,
\end{equation}
and, the noisy model
\begin{align}\label{eq:projection integral_noisy}
\begin{aligned}
P'_{i}(x,y) &= \int_{-\infty}^{\infty} \phi_k \left( R_{i}{r}
\right)\, dz , \quad r=(x,y,z)^{T},\quad i=1,\ldots,N,
\\
P_{i}(x,y) &= P'_{i}(x+\Delta x_i,y+\Delta y_i) + \mbox{noise},
\end{aligned}
\end{align}
where we ignore the CTF. This is in accordance with the common practice that ignores the CTF when estimating an initial model, and only applies phase-flipping to the images~\cite{cong2010single}. If an image $P_{i}$ was generated from $\phi_{k}$ according to~\eqref{eq:projection integral_noisy}, then we denote $C(i)=k$, that is, the ``class'' of image $P_{i}$ is $k$.

We denote by $G_{k}$ the set of indices of all images that correspond to the molecule $\phi_{k}$, namely
\begin{equation}\label{eq:Gk}
G_{k} = \left \{ i\, | \, C(i)=k \right \}, \quad k=1,\ldots,K.
\end{equation}
Thus, $G_{1},\ldots,G_{K}$ are disjoint sets whose union is equal to the set $\left \{ 1,\ldots, N \right \}$. Our goal is to estimate the sets $G_{k}$, $k=1,\ldots,K$, and the rotations $R_{i}$, $i=1,\ldots,N$ of~\eqref{eq:projection integral}, given only the images $P_1,\ldots, P_N$. Once these sets and rotations have been estimated, the structures $\phi_{1},\ldots,\phi_{K}$ can be reconstructed using standard algorithms~\cite{Herman09,ReconMethods01}.

In this work we assume that $K$ is known. This is a common assumption in many works addressing the heterogeneity problem~\cite{shatsky2010automated,scheres2012relion} and in all existing software packages. Some works, such as~\cite{anden2015covariance,katsevich2013covariance}, may allow to detect $K$ automatically. A consequence of the dependency on $K$ is that the heterogeneity is assumed to be discrete, namely, that the underlying structure assumes one of $K$ possible conformations. In contrast to this approach, some recent works try to tackle continuous heterogeneity, where the underlying structure exhibits a continuum of conformations~\cite{lederman2017continuously}.

\section{Partitioning a graph: max-cut} \label{sec:preliminaries}
In this section, we present the max-cut problem, its Goemans-Williamson approximation algorithm, and a particular case for which this algorithm is exact. This particular case is used later for the analysis of our algorithm.

Let $(V,E)$ be a weighted graph, with a set of vertices $V$ and a set of edges $E$. We denote the nodes in $V$ by $v_{1},\ldots,v_{N}$, hence $\left \lvert V \right \rvert=N$. We assume that each edge $(v_{i},v_{j}) \in E$ is associated with a real-valued positive weight $w(v_{i},v_{j})$. Whenever $(v_{i},v_{j}) \not \in E$, we set $w(v_{i},v_{j})=0$. The weighted adjacency matrix $W$ of the graph $(V,E)$ is a matrix of size $N \times N$ with entries $w_{ij} = w(v_{i},v_{j})$. For simplicity of notation, we identify the set $\{v_1, \ldots, v_N\}$ with the set $\{1, \ldots ,N\}$ (via the trivial map $v_{i} \mapsto i$). A cut (sometimes called a 2-cut) is a partition of the set of vertices $V$ into two disjoint subsets, namely, into $G_{1}$ and $G_{2}$ such that $G_1 \cup G_2 = \{1, \ldots, N\}$ and $G_1 \cap G_2 = \emptyset$. The weight of a cut is defined as
\begin{equation}\label{eq:cut weight}
W(G_{1},G_{2}) = \sum_{i \in G_{1}, j \in G_{2}} w_{ij}.
\end{equation}
The maximum-cut (max-cut) problem is to find a cut whose weight is maximal among all possible cuts.
Although this problem has been proven to be NP-complete, it has many approximation algorithms. The currently best approximation algorithm (and assuming the unique games conjecture~\cite{khot2005unique,goemans1995improved}, also the best approximation algorithm possible with polynomial complexity) is a randomized algorithm by Goemans and Williamson~\cite{goemans1995improved}. This algorithm guarantees that the value of its returned cut is at least 0.87 of the optimal result with probability as high as required. For a precise formulation of the result see~\cite{goemans1995improved}.

There is an analogous formulation of the max-cut problem for partitioning a graph into $K$ sub-graphs, which is called max $K$-cut. In this case, a $K$-cut is a partition of $V$ into $K$ disjoint sets, that is, into the sets $G_{1},\ldots,G_{K}$ such that
\begin{equation*}
\bigcup_{k=1}^{K} G_{k}=V, \quad G_{i} \cap G_{j} = \emptyset,\qquad i \ne j.
\end{equation*}
The weight of the cut in this case is defined as
\begin{equation}\label{eq:kcut}
W(G_{1},\ldots,G_{K}) = \sum_{k=1}^{K} \sum_{i\in G_{k}, j\not \in G_{k}} w_{ij},
\end{equation}
that is, the weight of the cut is the sum of all the edges whose endpoint vertices are in two different sets.
Similarly to the Goemans-Williamson algorithm, there are algorithms~\cite{frieze1997improved} that can be applied in this case. The results in~\cite{frieze1997improved} also give lower bounds on the  approximation for different values of $ K $. For some large values of $ K $ the bounds are better than the 0.87 ratio of the $ K=2 $ case (e.g., as given in~\cite{frieze1997improved}, for $ K=10 $ the approximation ratio is better than $ 0.92 $).

Although the Goemans-Williamson algorithm finds only an approximate solution, there is a special family of graphs for which its solution is exact.
\begin{definition}\label{def:bipartite}
The graph $(V,E)$ is bipartite if its vertices can be partitioned into two disjoint sets such that every edge connects a vertex in the first set to a vertex in the second set.
\end{definition}
We next show in Lemma~\ref{lem:maxCut_twoSide} that for a bipartite graph, the Goemans-Williamson algorithm finds the optimal cut (and not an approximation of it). This lemma is used later to analyze the properties of our algorithm.
\begin{lemma} \label{lem:maxCut_twoSide}
The Goemans-Williamson algorithm finds the exact solution for the max-cut problem of a bipartite graph.
\end{lemma}
The proof of this lemma is based on some detailes of the Goemans-Williamson algorithm, and is given in Appendix~\ref{sec:proof_lem_twosided}.

\section{Algorithm description}\label{sec:alg}

Let $P_{i}$ and $P_{j}$ be two images generated according to~\eqref{eq:projection integral} (using the same $\phi$), and let the $ n $-dimensional Fourier transform $ \mathcal{F} $ be defined by
\begin{equation*}
\hat{f}(\xi) = \mathcal{F}(f)(\xi)=\int_{\R^n} e^{-i x \cdot \xi}f(x)dx, \quad \xi \in \mathbb{R}^{n}.
\end{equation*}
If we denote by $\hat{P}_{i}$ the two-dimensional Fourier transform of $P_{i}$ and by $\hat{\phi}$ the three-dimensional Fourier transform of $\phi$, then,  the Fourier projection-slice theorem~\cite{Natterer} implies that  $\hat{P}_i$ is the restriction of $\hat{\phi}$ to the plane spanned by the first two columns of $R_{i}$ of~\eqref{eq:projection integral}. Explicitly,
\begin{equation}\label{eq:slice theorem}
\hat{P}_{i}(\omega_{x},\omega_{y}) = \hat{\phi} \left (\omega_{x}
R_{i}^{(1)} + \omega_{y} R_{i}^{(2)} \right ),
\end{equation}
where $R_{i}^{(1)}$, $R_{i}^{(2)}$,
$R_{i}^{(3)}$ are the columns of the rotation matrix $R_{i}$. As a consequence of~\eqref{eq:slice theorem}, any two (Fourier-transformed) images $\hat{P}_i$ and $\hat{P}_j$ share a common line through the origin, namely, there exist unit vectors $c_{ij},c_{ji}\in\mathbb{R}^{2}$ such that $\hat{P}_i(\xi c_{ij}) = \hat{P}_j(\xi c _{ji})$ for any $\xi\in\mathbb{R}$. The vectors $c_{ij}$ and $c_{ji}$ are given explicitly~\cite{cryoeig} by
\begin{equation}\label{eq:cl equation}
c_{ij} = \left( \begin{array}{ccc} 1 & 0 & 0 \\ 0 & 1 & 0 \end{array} \right) R_{i}^T \frac{R_{i}^{(3)} \times R_{j}^{(3)}}{\|R_{i}^{(3)} \times R_{j}^{(3)}\|}, \quad
c_{ji} = \left( \begin{array}{ccc} 1 & 0 & 0 \\ 0 & 1 & 0 \end{array} \right) R_{j}^T \frac{R_{i}^{(3)} \times R_{j}^{(3)}}{\|R_{i}^{(3)} \times R_{j}^{(3)}\|}.
\end{equation}
A simple method to estimate the common lines from the Fourier transform of projection images~\eqref{eq:projection integral_noisy} is to find the two radial lines in the transformed images (passing through the origin) which have the highest correlation. This procedure is described in~\cite{singer2012center}. It is also described in~\cite{singer2012center} how to modify this procedure to handle the unknown shifts in~\eqref{eq:projection integral_noisy}.

If we lift the vectors $c_{ij}$ and $c_{ji}$ to $\mathbb{R}^{3}$ by zero padding, then it can be shown that for all $i$ and $j$ it holds that $R_{i} c_{ij} = R_{j} c_{ji}$ (see~\cite{Shkolnisky2012} for a detailed proof). Now, assume that we are given rotations $\tilde{R}_{i}$ and $\tilde{R}_{j}$, which are estimates of $R_{i}$ and $R_{j}$, as well as vectors $\tilde{c}_{ij}$ and $\tilde{c}_{ji}$, which are estimates of $c_{ij}$ and $c_{ji}$. Then, we define a similarity score for any pair of images $i$ and $j$ by $\|\tilde{R}_i \tilde{c}_{ij} - \tilde{R}_j \tilde{c}_{ji}\|$.
This score is $0$ if the common lines and the rotations are correct, and is small for small errors in the common lines or in the rotations.
The Least Unsquared Deviations (LUD) algorithm~\cite{wang2013orientation} finds rotations $\tilde{R}_i$, $i=1,\ldots,N$, that bring the score $ \sum_{i,j}\|\tilde{R}_{i} \tilde{c}_{ij} - \tilde{R}_{j} \tilde{c}_{ji}\|$ to a local minimum given only the common lines $ \tilde{c}_{ij}$. 

As the algorithm~\cite{wang2013orientation} assumes the homogeneous setting, namely, that all images were generated from the same underlying $\phi$, it cannot be directly applied to the heterogeneous setting. However, it can be applied to each of the (unknown) sets $G_{k}$ of~\eqref{eq:Gk}.

Consider a graph $(V,E)$ whose $i$'th vertex corresponds to the image $P_{i}$, and where any two vertices are connected by an edge. Our goal is to partition the vertices of the graph $(V,E)$ into the sets $G_{k}$ of~\eqref{eq:Gk}. Assume moreover that we are given weights for the edges $E$, and denote by $w_{ij}$ the weight of the edge between vertices $i$ and $j$. In this case, according to~\eqref{eq:kcut}, the weight of some cut $\tilde{G} = \left \{\tilde{G}_{1},\ldots,\tilde{G}_{K} \right \}$ is equal to
\begin{equation}\label{eq:WS}
\begin{aligned}
W(\tilde{G}) &= \sum_{k=1}^{K} \sum_{i\in \tilde{G}_{k}, j\not \in \tilde{G}_{k}} w_{ij} = \sum_{k=1}^{K} \left (\sum_{\substack{j=1 \\ i\in \tilde{G}_{k}}}^{N} w_{ij} - \sum_{i,j \in \tilde{G}_{k}} w_{ij} \right )\\
& = \sum_{i,j=1}^{N} w_{ij} - \sum_{k=1}^{K}\sum_{i,j\in \tilde{G}_{k}} w_{ij} = C - \sum_{k=1}^{K}\sum_{i,j\in \tilde{G}_{k}} w_{ij},
\end{aligned}
\end{equation}
where $C$ is the sum of all weights in the graph (independent of the partition). Thus, finding the cut that maximizes $W(\tilde{G})$ is equivalent to finding the cut that minimizes $\sum_{k=1}^{K}\sum_{i,j\in \tilde{G}_{k}} w_{ij}$. If we now set
\begin{equation}\label{eq:LUDscore}
w_{ij} = \|R_i c_{ij} - R_j c_{ji}\|,
\end{equation}
we get that $W(\tilde{G}) \ge 0$ for any partition $\tilde{G}$, and that $W(G_{1},\ldots,G_{K})=0$. In other words, the rotations $R_{1},\ldots,R_{N}$ and the partition $G=\left\{G_{1},\ldots,G_{K}\right\}$ are obtained as a solution to the optimization problem
\begin{equation}\label{eq:heteroScore}
\min_{\tilde{G},\tilde{R}} F(\tilde{G},\tilde{R}),
\end{equation}
where $ F $ is defined by
\begin{equation}\label{eq:Fdef}
F(\tilde{R},\tilde{G})=\sum_{k=1}^{K}\sum_{i,j\in \tilde{G}_{k}} \|\tilde{R}_i \tilde{c}_{ij} - \tilde{R}_j \tilde{c}_{ji}\|.
\end{equation}

The formulation~\eqref{eq:heteroScore} is used since it results in a generalization of a proven technique (namely, the LUD) to the heterogeneous case. Moreover, the resulting algorithm for the heterogeneous case uses all images at once to estimate a partition. Note however, that the choice of weights given by \eqref{eq:LUDscore} is not unique. For example, it can be replaced with any rotation invariant metric that is a function of $R_ic_{ij}$ and $R_jc_{ji}$, with only minor changes to subsequent proofs. One such metric is the squared distance which is related to the algorithms in~\cite{Shkolnisky2012, cryoeig}. We choose the LUD score in this paper since it uses only the common lines (and no prior model of the molecule), and it is robust to outliers~\cite{wang2013orientation}. Similarly, there are other possibilities for the objective function~\eqref{eq:heteroScore}, such as the maximum-likelihood score. The main advantage of \eqref{eq:heteroScore} is that it is amenable to a rigorous mathematical analysis due to its relations with the max-cut problem.

Below, we first propose a ``theoretical" algorithm that assumes to get as an input the estimated rotations $\tilde{R}_{i}$ of~\eqref{eq:projection integral_noisy} (approximations of the rotations $R_{i}$ of~\eqref{eq:projection integral_noisy}). This assumption is common in many existing algorithms, including~\cite{katsevich2013covariance,shatsky2010automated}. Given the estimated rotations and the estimated common lines, Algorithm~\ref{alg:theoretical} below builds the corresponding graph, and returns the classification found by the max-$ K $-cut algorithm.
Note that the shifts $ (\Delta x_i, \Delta y_i,0) $ in \eqref{eq:projection integral_noisy} affect Algorithm~\ref{alg:theoretical} only through the procedure of detecting common lines, and are handled as described in~\cite{singer2012center}. In Section~\ref{sec:proofs} we provide theoretical guarantees for this algorithm.

\begin{algorithm}
	\caption{Classification of heterogeneous data sets from common lines and rotations}
	\label{alg:theoretical}
	\begin{algorithmic}[1]
		\State {\bfseries Input:}\begin{tabular}[t]{ll}
			$\{\tilde{c}_{ij}\}_{i,j=1}^N$ & Estimated common lines between all images $P_{1},\ldots,P_{N}$.\\
			$\tilde{R}_{1},\ldots,\tilde{R}_{N}$ & Estimated rotations of all the images.\\
			$K$ & Number of groups.
		\end{tabular}
		\State {\bfseries Output:}\begin{tabular}[t]{ll}
			$\tilde{G}_{1},\ldots, \tilde{G}_{K}$ &  Estimated partition of the images into homogeneous groups.
		\end{tabular}
		\For {$i=1$ to $N$}
		\For {$j=1$ to $N$}
		\State $W_{ij} \gets \|\tilde{R}_i \tilde{c}_{ij} - \tilde{R}_j c_{ji} \|$ \Comment{$W$ is the graph to partition.}
		\EndFor
		\EndFor
		\State $\left [ \tilde{G}_{1},\ldots, \tilde{G}_{K}\right ] \gets \operatorname{max-K-cut}(W,K)$ \Comment{Apply max-$K$-cut.}\label{algstep:maxcut}
		\State $\tilde{G} =\{\tilde{G}_1, \ldots, \tilde{G}_K\}$
		\State\Return $\tilde{G}_{1},\ldots,\tilde{G}_{K}$
	\end{algorithmic}
\end{algorithm}
Next we bound the computational complexity of Algorithm~\ref{alg:theoretical}.
In Algorithm~\ref{alg:theoretical} we first construct the weights matrix $ W $ which costs $ \mathcal{O}(N^2) $ operations, and then solve the max-cut problem. The complexity of the Goemans-Williamson algorithm is dominated by the complexity of its underling SDP problem. Commonly used interior point methods have complexity of $\tilde{\mathcal{O}}(N^{3.5}) $~\cite{alizadeh1995interior}. Thus, the complexity of Algorithm~\ref{alg:theoretical} is $ \tilde{\mathcal{O}}(N^{3.5}+N^2) = \tilde{\mathcal{O}}(N^{3.5}) $.

Since the assumption of having accurate estimates $\tilde{R}_{i}$ does not always hold, and moreover, the optimization problem~\eqref{eq:heteroScore} is high-dimensional and non-convex, we propose the following descent procedure to find its minimum in practice. Start from an initial estimate for the classification $G$, and minimize~\eqref{eq:heteroScore} by minimizing alternatingly over the set of rotation $\tilde{R}=\{\tilde{R_i}\}_{i=1\ldots N}$ and the classification of the images $\tilde{G}$ as follows:
\begin{enumerate}
\item Find the minimizing $\tilde{R}$ for a given partition $\tilde{G}$ (known from the previous iteration), by applying the LUD algorithm~\cite{wang2013orientation} on each $\tilde{G}_{k}$, $k=1,\ldots,K$.
\item Given the new rotations, find a partition $\tilde{G}_1, \ldots, \tilde{G}_K$. This can be achieved approximately using the Frieze and Jerrum algorithm for Max-$K$-Cut~\cite{frieze1997improved} (or, for $ K =2 $, using the  Goemans-Williamson algorithm~\cite{goemans1995improved}).
\end{enumerate}
Once the rotations $\tilde{R}$ have been estimated in step 1, minimizing~\eqref{eq:WS} with weights given by~\eqref{eq:LUDscore} in step 2 is equivalent to maximizing~\eqref{eq:heteroScore}. Note that the objective in~\eqref{eq:WS} is exactly the one maximized by the solution to the max-cut problem applied to the graph $(V,E)$ defined above, with weights given by~\eqref{eq:LUDscore}. Optimization algorithms for steps 1 and 2 (minimizations over the rotations and the partitions) provide only approximate solutions, and thus, the solution to~\eqref{eq:heteroScore} is also only approximate. Nevertheless, we show later (Theorem~\ref{thm:aqu}) that in some cases we can bound the quality of this approximate solution. A detailed description of the algorithm is given in Algorithm~\ref{alg:rot_maxcut}.

As will be demonstrated in Section~\ref{sec:experimental_results}, Algorithm~\ref{alg:rot_maxcut} favors balanced classes. The reason for that is as follows. For a balanced partition, the number of edges in the induced cut is larger than for unbalanced partitions. For high levels of noise, the variability of the weights of the edges becomes smaller (loosely speaking meaning that the weights become ``similar to each other''). Thus, the score of the cut increases with the number of edges in it, which in turn implies favoring roughly balanced splits of the nodes.

\begin{algorithm}
	\caption{Reconstruction from heterogeneous data sets}
	\label{alg:rot_maxcut}
	\begin{algorithmic}[1]
	\State {\bfseries Input:}\begin{tabular}[t]{ll}
                                $\{\tilde{c}_{ij}\}_{i,j=1}^N$ & Estimated common lines between all images $P_{1},\ldots,P_{N}$.\\
                                $K$ & Number of groups.
                             \end{tabular}
	\State {\bfseries Output:}\begin{tabular}[t]{ll}
                                $\tilde{R}_{1},\ldots,\tilde{R}_{N}$ & Estimated rotations of all the images.\\
                                $\tilde{G}_{1},\ldots, \tilde{G}_{K}$ &  Estimated partition of the images into homogeneous groups.
                               \end{tabular}
		\Statex \(\triangleright\) \underline{Initialization}\medskip
        \State $\tilde{G}^{(0)}_{1} \gets \left \{1,\ldots,N\right \}$ \Comment{Start from an initial guess. For example, all images in $\tilde{G}^{(1)}_{1}$.}
        \For {$k=2$ to $K$}
        \State $\tilde{G}_{k}^{(0)} \gets \emptyset$
        \EndFor
        \State $n\gets 0$ \Comment{Iteration number.}
		\Repeat
            \State $n\gets n+1$
	        \Statex \(\triangleright\) \underline{LUD step}\medskip
            \For {$k=1$ to $K$}
            \State $C_{k}^{(n)} = \left \{ \tilde{c}_{ij} \ | \  i \in \tilde{G}_{k}^{(n-1)} \wedge j \in \tilde{G}_{k}^{(n-1)}\right \}$ \Comment{Construct common lines matrix for $\tilde{G}_{k}^{(n-1)}$.}
            \State $\{\tilde{R}_{i}^{(n)} |~ i \in \tilde{G}_{k}^{(n-1)}\} = \operatorname{LUD}(C_{k}^{(n)})$ \Comment{Find rotations for images in $\tilde{G}_{k}^{(n)}$ using LUD~\cite{wang2013orientation}.} \label{algstep:lud}
            \EndFor
            \State $\tilde{R}^{(n)} = \{\tilde{R}^{(n)}_1, \ldots, \tilde{R}^{(n)}_N\} $
            \If {$F(\tilde{R}^{(n)}, \tilde{G}^{(n-1)}) > F(\tilde{R}^{(n-1)}, \tilde{G}^{(n-1)})$} \Comment {$F$ is defined in \eqref{eq:Fdef}.}
	            \State {$\tilde{R}^{(n)} = \tilde{R}^{(n-1)} $}
            \EndIf

        \Statex \(\triangleright\) \underline{max-cut step}\medskip
        \State $\left [ \tilde{G}_{1}^{(n)},\ldots, \tilde{G}_{K}^{(n)}\right ] \gets $ Algorithm\_\ref{alg:theoretical}$ (\{\tilde{c}_{ij}\},\tilde{R}^{(n)}) $ \label{algstep:alg1_in _alg2}
        \State $\tilde{G}^{(n)} =\{\tilde{G}^{(n)}_1, \ldots, \tilde{G}^{(n)}_K\}$
        \If {$F(\tilde{R}^{(n)}, \tilde{G}^{(n)}) > F(\tilde{R}^{(n)}, \tilde{G}^{(n-1)})$}
	        \State {$\tilde{G}^{(n)} = \tilde{G}^{(n-1)} $}
        \EndIf
		\Until{$\left|F(\tilde{R}^{(n)}, \tilde{G}^{(n)}) - F(\tilde{R}^{(n-1)}, \tilde{G}^{(n-1)})\right| \leq \delta$
}  \Comment{$\delta$ is a small constant.}

        \State\Return $\tilde{R}_{1}^{(n)},\ldots,\tilde{R}_{N}^{(n)}$ and $\tilde{G}_{1}^{(n)},\ldots,\tilde{G}_{K}^{(n)}$
	\end{algorithmic}
\end{algorithm}

\section{Convergence and error bounds}\label{sec:proofs}
For simplicity of the presentation, we assume in this section that $K = 2$, as all arguments are easily extended to any $K$. We will show that Algorithm~\ref{alg:rot_maxcut} stops in finite time, and that it is stable around its optimum, namely, that for ``good" initial conditions, Algorithm~\ref{alg:theoretical} finds an accurate class partition with high probability, or, equivalently, that when the initial partition and rotations are close to the optimum, step~\ref{algstep:alg1_in _alg2} of Algorithm~\ref{alg:rot_maxcut} finds an accurate class partition with high probability.

We start by establishing the finite-time convergence of Algorithm~\ref{alg:rot_maxcut}.
\begin{theorem} \label{thm:conv_to_min}
The score $ F $ in Algorithm~\ref{alg:rot_maxcut} converges monotonically to a minimum or to a saddle point.
\end{theorem}
\begin{proof}
Since in each iteration of Algorithm~\ref{alg:rot_maxcut} we do not increase the value of $F(R,G)$ from~\eqref{eq:Fdef}, we have that $F(R^{(n)},G^{(n)}) \geq F(R^{^{(n+1)}},G^{^{(n+1)}})$, and thus $F(R^{(n)},G^{(n)})$ is monotonically non-increasing and bounded by zero.
\end{proof}

Theorem~\ref{thm:conv_to_min} shows that Algorithm~\ref{alg:rot_maxcut} converges, and moreover, that its underlying score~\eqref{eq:Fdef} is monotonically non-increasing. Since the stopping criteria of Algorithm~\ref{alg:rot_maxcut} is that the difference between the scores in successive  steps is smaller than some value, this means that the algorithm terminates in finite time. Note that the convergence of the score $F$ of~\eqref{eq:Fdef} does not necessarily imply that the rotations~$R$ and the partition~$G$ converge. For example, the rotations and the partition are not even unique for a given score, as the rotations are unique only up to a global rotation and handedness, and any permutation of the classes in the partition will give the same score.

Next, we derive error bounds for Algorithm~\ref{alg:theoretical}. We start by proving that if the common lines are detected correctly, and the rotations are assigned correctly, then the algorithm finds the correct partition $\left \{ G_{1},G_{2} \right \}$.

\begin{theorem}
Suppose that the rotations $\tilde{R}_{1},\ldots,\tilde{R}_{N}$ and the common lines $\tilde{c}_{ij}\in\mathbb{R}^{2}$ in the input of
Algorithm~\ref{alg:theoretical} are correct when the images $P_{i}$ and $P_{j}$ are in the same class, and are uniformly distributed otherwise. Then Algorithm~\ref{alg:theoretical} will find the correct class partitioning.
\end{theorem}
\begin{proof}
For the correct rotations $R_{1},\ldots,R_{N}$, we have that for any two images $P_{i}$ and $P_{j}$ in the same class it holds that $\|R_i c_{ij} - R_j c_{ji} \| = 0$. For $P_{i}$ and $P_{j}$ in different classes, the score $\|R_i c_{ij} - R_j c_{ji} \|$ is a random variable that gets the value $0$ with probability $0$. Thus, if we consider the graph whose adjacency matrix $W$ is given by $W_{ij}= \|R_i c_{ij} - R_j c_{ji} \|$, we get a bipartite graph, since all images from the same class are not connected (connected with weight $0$) and images from different classes are connected with some random weight. By Lemma~\ref{lem:maxCut_twoSide}, the Goemans-Williamson algorithm returns the correct partition for this graph.
\end{proof}

In order to show that for common lines and rotations ``close" to the correct ones Algorithm~\ref{alg:theoretical} still finds a partition that is ``close'' to the correct one, we need to define what are ``close" common lines and rotations, as well as what are ``close'' partitions.
\begin{definition}\label{def:cl_dist}
Let $c_{ij}$ and $c_{ji}$ be the correct common line between images $P_i$ and $P_j$, as defined by~\eqref{eq:cl equation}. Let $\tilde{c}_{ij}$ and $\tilde{c}_{ji}$ be some estimates of $c_{ij}$ and $c_{ji}$, respectively.  We define the distance between $c_{ij}$ and $\tilde{c}_{ij}$ and between $c_{ji}$ and $\tilde{c}_{ji}$ as the angle between them, namely,
\begin{equation}
\begin{array}{c}
d(c_{ij}, \tilde{c}_{ij}) = \arccos \ip{c_{ij}}{\tilde{c}_{ij}},\\
d(c_{ji}, \tilde{c}_{ji}) = \arccos \ip{c_{ji}}{\tilde{c}_{ji}},
\end{array}
\end{equation}
where $\ip{\cdot}{\cdot}$ is the standard dot product in $\mathbb{R}^{2}$ (and $ \arccos  $ is measured in radians).
\end{definition}

Unless otherwise stated, all matrix norms below refer to the induced 2-norm $\left \lVert \cdot \right \rVert_{2}$.
\begin{lemma}\label{lem:tri_so3}
	If $\|\tilde{R}_1 - R_1\|<\varepsilon$ and $\|\tilde{R}_2 -  R_2\|<\varepsilon$, then $\|\tilde{R}_1\tilde{R}_2 - R_1 R_2\|<2\varepsilon$.
\end{lemma}
\begin{proof}
From the triangle inequality and the fact that $\| R\|=1$ for any $R \in SO(3)$, we get that $ \| \tilde{R}_1\tilde{R}_2 - R_1 R_2 \| = \| \tilde{R}_1\tilde{R}_2 - \tilde{R}_1 R_2 + \tilde{R}_1 R_2 - R_1 R_2 \| \leq  \| \tilde{R}_1 (\tilde{R}_2 - R_2) \|+ \|(\tilde{R}_1 - R_1 ) R_2 \| <2\varepsilon$.
\end{proof}
Next, we define the quality of a partition of a heterogeneous data set.
\begin{definition}\label{def:good_class}
A partition of $P_{1},\ldots,P_{N}$ into the $K$ sets $\tilde{G}_{i}, i = 1,\ldots, K$, is called $p$-precise for class $k$, $k=1, \ldots, K$, if $\abs{\tilde{G}_{k} \cap G_{k}}/\abs{\tilde{G}_{k}} \ge p$, where $G_{k}$ is defined in~\eqref{eq:Gk}, and $\abs{G}$ is the number of elements in the set $G$. A partition is $p$-precise if there is a permutation of the indexes $ 1,\ldots, K  $ such that it is $p$-precise for all $k$.
\end{definition}
In other words, a partition of the images $P_{1},\ldots,P_{N}$ is $p$-precise if the ratio between the number of correct images in class $k$ and the total number of images assigned to class $k$ is at least $p$. Note that the correct partition is $1$-precise, and a random partition is about $\frac{1}{K}$-precise.

Theorem~\ref{thm:aqu} below shows that Algorithm~\ref{alg:theoretical} gives ``good'' results if started from a ''good'' initial state. For ease of notation, Theorem~\ref{thm:aqu} is proven for the case $K=2$, and we moreover assume that the correct partition satisfies $\abs{G_{1}}=\abs{G_{2}}=N/2$.

\begin{theorem}\label{thm:aqu}
Let $P_{1},\ldots,P_{N}$ be $N$ images comprising a heterogeneous data set corresponding to $K=2$. Let $c_{ij}$ and $c_{ji}$ be the common lines between $P_{i}$ and $P_{j}$, and let $\tilde{c}_{ij}$ and $\tilde{c}_{ji}$ be some estimate of $c_{ij}$ and $c_{ji}$ used as the input to Algorithm~\ref{alg:theoretical}. Also, let $R_{i}$ be the rotation corresponding to $P_{i}$ (see~\eqref{eq:projection integral_noisy}), and let $\tilde{R}_{i}$ be the rotation used as the input of Algorithm~\ref{alg:theoretical}. Let $\varepsilon>0$ and assume that
\begin{enumerate}
\item $\|R_{i} - \tilde{R}_{i}\| \leq \varepsilon$, $i=1,\ldots,N$.
\item If $C(i)=C(j)$ then $d(c_{ij},\tilde{c}_{ij}) \leq \varepsilon$, $i,j=1,\ldots,N$, $i \neq j$.
\item If $C(i)\not = C(j)$ then $\tilde{c}_{ij}$, $i,j=1,\ldots,N$, $i \neq j$, is uniformly distributed and independent of the rotations.
\end{enumerate}
Then, for a sufficiently large $N$, with high probability, Algorithm~\ref{alg:theoretical} results in at least $0.87-3.555\varepsilon$-precise partition (according to Definition~\ref{def:good_class}).
\end{theorem}

Assumption 1 in Theorem~\ref{thm:aqu} states that Algorithm~\ref{alg:theoretical} is initialized with rotations that are close to the true ones. Assumption 2 states that common lines between images of the same class are detected with small error. Assumption 3 states that common lines between images of different classes (that therefore have no common line) are uniformly random. This  uniformity assumption
enables us to put a rigorous bound on the classification error. In real world applications, this assumption is not necessarily valid.  In case of
non-uniform common lines errors, the proof will follow similarly, provided that it is
possible to bound the expected value of the score between two images of  different
classes, and show that it is not ``too small".

The proof of Theorem~\ref{thm:aqu} consists of the following steps:
\begin{enumerate}
	\item Give an upper bound for the weight $w_{ij}$ (defined in the beginning of Section~\ref{sec:preliminaries}) for two images of the same class (Lemma~\ref{lem:weight_same_class} below).
	\item Estimate the distribution of the weight $w_{ij}$ for two images of different classes.
	\item Show that the score~\eqref{eq:Fdef} for the correct partition (with the given $ \tilde{c}_{ij} $ and $ \tilde{R}_i $) cannot be too small.
	\item For any given cut, estimate the score as a function of the precision of the cut.
	\item Estimate the maximal score~\eqref{eq:Fdef} over all ``bad" partitions.
	\item Show that with high probability the maximal score over all ``bad" partitions is lower then $ 0.87 $ times the score of the correct partition.
	\item Since Goemans-Williamson algorithm guarantees a partition with score higher than 0.87 times the score of the correct partition, we deduce that the algorithm cannot return a ``bad" partition.
\end{enumerate}

\begin{lemma}\label{lem:weight_same_class}
	Following the assumptions of Theorem~\ref{thm:aqu}, suppose that $P_{i}$ and $P_{j}$ are in the same class. Then, for
	\begin{equation}\label{eq:epsilon_ij_def}
	\epsilon_{ij}=\|\tilde{R}_{i} \tilde{c}_{ij} - \tilde{R}_{j} \tilde{c}_{ji} \|,
	\end{equation}
	it holds that $0 \leq \epsilon_{ij} \leq 4 \varepsilon$
	
\end{lemma}
\begin{proof}
Due to the invariance of the 2-norm to orthogonal transformations, we have that
	\begin{equation}\label{eq:eps_bound}
	\begin{aligned}
	\|\tilde{R}_i \tilde{c}_{ij} - \tilde{R}_j \tilde{c}_{ji} \|  & = \|\tilde{c}_{ij} - \tilde{R}_i^{-1} \tilde{R}_j \tilde{c}_{ji} \| \\
	& = \|\tilde{c}_{ij} - c_{ij} + c_{ij} - \tilde{R}_i^{-1} \tilde{R}_j \tilde{c}_{ji} + (\tilde{R}_i^{-1} \tilde{R}_j c_{ji} - \tilde{R}_i^{-1} \tilde{R}_j c_{ji}) \| \\
	& \leq \|\tilde{c}_{ij} - c_{ij} \| + \| c_{ij} - \tilde{R}_i^{-1} \tilde{R}_j \tilde{c}_{ji}  + \tilde{R}_i^{-1} \tilde{R}_j c_{ji} - \tilde{R}_i^{-1} \tilde{R}_j c_{ji} \| \\
	& \leq \| \tilde{c}_{ij} - c_{ij} \| + \| c_{ij} - (\tilde{R}_i^{-1} \tilde{R}_j) (R_i^{-1} R_j)^{-1} (R_i^{-1} R_j) c_{ji} \|
	+ \|\tilde{R}_i^{-1} \tilde{R}_j (\tilde{c}_{ji} - c_{ji}) \| \\
	& = \| \tilde{c}_{ij} - c_{ij} \| + \| c_{ij} - (\tilde{R}_i^{-1} \tilde{R}_j) (R_i^{-1} R_j)^{-1} c_{ij} \| + \|\tilde{c}_{ji} - c_{ji} \| \\
	& \leq 4 \varepsilon,
	\end{aligned}
	\end{equation}
	where the last inequality follows from assumptions 1 and 2 in Theorem~\ref{thm:aqu}, together with Lemma~\ref{lem:tri_so3} and the fact that $ R_i^{-1}R_jc_{ji} = c_{ij} $.
\end{proof}

\begin{proof}[ proof of Theorem~\ref{thm:aqu}]
Let $(V,E)$ be a graph whose vertex $v_{i}\in V$ corresponds to the image $P_{i}$, $i=1,\ldots,N$, and whose (undirected) edge $(v_{i},v_{j})\in E$ has weight $\|\tilde{R}_i \tilde{c}_{ij} - \tilde{R}_j \tilde{c}_{ji} \|$. A score of a partition is the value of the function $ F $ of \eqref{eq:Fdef}, that is, the sum of the weights of all the edges that connect nodes from one class to nodes from other classes.

For $P_{i}$ and $P_{j}$ in the same class, namely, $C(i)=C(j)$,  by Lemma~\ref{lem:weight_same_class} and using the notation in~\eqref{eq:epsilon_ij_def}, we have that $0 \leq \epsilon_{ij} \leq 4 \varepsilon$.
	
Next, if $C(i)\ne C(j)$, we denote $X_{ij}=\|\tilde{R}_{i} \tilde{c}_{ij} - \tilde{R}_{j} \tilde{c}_{ji} \|$. Also, we denote $X'_{ij}=\|R_{i} \tilde{c}_{ij} - R_{j} \tilde{c}_{ji} \|$, that is, $X'_{ij}$ is defined using the correct rotations.
Note that both $X_{ij}$ and $X'_{ij}$ are random variables. Intuitively, we are going to show that $X'_{ij}$ is ``much larger'' than $\epsilon_{ij}$ with high probability, and that $X_{ij}$ is ``close'' to $X'_{ij}$. Thus, to maximize the cut of the graph $(V,E)$, images $P_{i}$ and $P_{j}$ of the same class (satisfying $C(i)=C(j)$) for which $\epsilon_{ij}$ is small should be assigned to the same subset of the partition.

Since we assume that the common lines $\tilde{c}_{ij}$ and $\tilde{c}_{ji}$ for $i$ and $j$ such that $C(i)\ne C(j)$ are uniformly random and independent of the rotations, the weights $X'_{ij}$ are i.i.d and distributed as the distance between two random unit vectors in $\mathbb{R}^3$, whose distribution is analyzed in~\cite{solomon1978geometric}.

To bound $X_{ij}$, we start by noting that
\begin{equation}\label{eq:Xij_more_terms}
X_{ij} = \|\tilde{R_i}\tilde{c}_{ij} - \tilde{R_j}\tilde{c}_{ji}\| = \|\tilde{R_i}\tilde{c}_{ij} - R_i \tilde{c}_{ij} + R_i \tilde{c}_{ij} - R_j \tilde{c}_{ji} + R_j \tilde{c}_{ji} - \tilde{R_j}\tilde{c}_{ji}\|,
\end{equation}

and that
\begin{multline*}
\| R_i \tilde{c_{ij}} - R_j \tilde{c}_{ji}\|  - \|\tilde{R_i}\tilde{c}_{ij} - R_i \tilde{c}_{ij}\| - \| R_j \tilde{c}_{ji} - \tilde{R_j}\tilde{c}_{ji}\|\leq\\
\underbrace{\| R_i \tilde{c}_{ij} - R_j \tilde{c}_{ji}  + \tilde{R_i}\tilde{c}_{ij} - R_i \tilde{c}_{ij} +  R_j \tilde{c}_{ji} - \tilde{R_j}\tilde{c}_{ji}\|}_{X_{ij} \mbox{ as in \eqref{eq:Xij_more_terms} with reordered terms}}
= \underbrace{\|\tilde{R_i}\tilde{c}_{ij} - R_i \tilde{c}_{ij} + R_i \tilde{c}_{ij} - R_j \tilde{c}_{ji} + R_j \tilde{c}_{ji} - \tilde{R_j}\tilde{c}_{ji}\|}_{X_{ij}\mbox{ as in \eqref{eq:Xij_more_terms}} } \\
\leq  \| R_i \tilde{c}_{ij} - R_j \tilde{c}_{ji}\|  + \|\tilde{R_i}\tilde{c}_{ij} - R_i \tilde{c}_{ij}\| + \| R_j \tilde{c}_{ji} - \tilde{R_j}\tilde{c}_{ji}\|,
\end{multline*}
where the first inequality is due to the reverse triangle inequality and the last is due to the triangle inequality.
From assumption 1 in Theorem~\ref{thm:aqu} it follows that $\|\tilde{R_i}\tilde{c}_{ij} - R_i \tilde{c}_{ij} \| = \|(\tilde{R_i} - R_i) \tilde{c}_{ij}\| \leq \|\tilde{R_i} - R_i\| \leq \varepsilon $,
and thus we have
\begin{equation}\label{eq:X_bound_X'}
	X'_{ij} - 2\varepsilon \leq X_{ij} \leq X'_{ij} + 2\varepsilon.
\end{equation}

Next, we analyze the score of an arbitrary partition of $V$, and show that with high probability the score of ``bad" partitions is low and of ``good" partitions is high. Thus, we get a bound for how ``bad" the partition generated by Algorithm~\ref{alg:theoretical} can be. Let $G_{1}$ and $G_{2}$ be the sets defined in~\eqref{eq:Gk}, and let $\tilde{G}_{1}$ and $\tilde{G}_{2}$ be some partition of the graph $(V,E)$. Also, recall that we assume that $K=2$ and that $\abs{G_{1}}=\abs{G_{2}}=N/2$. We denote $M=N/2$.
In our notation, $G_{1}$ and $G_{2}$ is the ground truth partition of the graph, and $\tilde{G}_{1}$ and $\tilde{G}_{2}$ is the cut returned by our algorithm. We denote
\begin{equation}\label{eq:Gkl}
G_{k,l} = \tilde{G}_{k} \cap G_{l}, \quad k,l=1,2,
\end{equation}
and so
\begin{equation}\label{eq:Gkl_details}
\begin{alignedat}{4}
 &\abs{G_{1,1}}=\abs{\tilde{G}_{1} \cap G_{1}} &=& p_{1}M, &\quad  &\abs{G_{1,2}}=\abs{\tilde{G}_{1} \cap G_{2}} &=& p_{2}M, \\
 &\abs{G_{2,1}}=\abs{\tilde{G}_{2} \cap G_{1}} &=& (1-p_{1})M, &\quad  &\abs{G_{2,2}}=\abs{\tilde{G}_{2} \cap G_{2}} &=& (1-p_{2})M,
\end{alignedat}
\end{equation}
where $0 \leq p_{1},p_{2} \leq 1$ and we have used the fact that $G_{1} \cup G_{2} = V$ and $G_{1} \cap G_{2} = \emptyset$. Note that $ p_1 = 1,~p_2 = 0 $ and $p_1 = 0,~p_2 = 1 $ corresponds to the correct partition. We thus get that
\begin{equation*}
\abs{\tilde{G}_{1}} = (p_{1}+p_{2})M, \quad \abs{\tilde{G}_{2}}= (2-p_{1}-p_{2})M.
\end{equation*}
Using the notation of~\eqref{eq:Gkl}, $\tilde{G}_{1}$ is our ``estimate'' for $G_{1}$, and the subsets $G_{1,1}$ and $G_{1,2}$ are the subsets of $\tilde{G}_{1}$ that were assigned ``correctly'' and ``incorrectly'', respectively. The case for $G_{2,1}$ and $G_{2,2}$ is analogous.

\begin{figure}
	\centering
	\begin{tikzpicture}
	\node at (0,3.9) {$\tilde{G}_{1}$};
	\node at (3,3.9) {$\tilde{G}_{2}$};
	\draw[densely dotted] (1.5,-0.5) -- (1.5,4);
	
	\draw [decorate,decoration={brace,amplitude=6pt},xshift=-4pt,yshift=0pt] (0.0,-0.1) -- (0.0,1.3) node [black,midway,xshift=-0.6cm] {$G_{1,1}$};
	\node at (0,0) {\textbullet};
	\node at (0,0.4) {\textbullet};
	\node at (0,0.9) {$\vdots$};
	\node at (0,1.2) {\textbullet};		
	\draw [decorate,decoration={brace,amplitude=6pt},xshift=-4pt,yshift=0pt] (0.0,1.9) -- (0.0,3.3) node [black,midway,xshift=-0.6cm] {$G_{1,2}$};
	\node at (0,2) {\textbullet};
	\node at (0,2.4) {\textbullet};
	\node at (0,2.9) {$\vdots$};
	\node at (0,3.2) {\textbullet};	
	
	\draw [decorate,decoration={brace,amplitude=6pt,mirror,raise=4pt},yshift=0pt] (3,-0.1) -- (3,1.3) node [black,midway,xshift=0.8cm] {$G_{2,1}$};		
	\node at (3,0) {\textbullet};
	\node at (3,0.4) {\textbullet};
	\node at (3,0.9) {$\vdots$};
	\node at (3,1.2) {\textbullet};	
	\draw [decorate,decoration={brace,amplitude=6pt,mirror,raise=4pt},yshift=0pt] (3,1.9) -- (3,3.3) node [black,midway,xshift=0.8cm] {$G_{2,2}$};		
	\node at (3,2) {\textbullet};
	\node at (3,2.4) {\textbullet};
	\node at (3,2.9) {$\vdots$};
	\node at (3,3.2) {\textbullet};	
	
	\draw [decorate,decoration={brace,amplitude=6pt},xshift=-4pt,yshift=0pt] (-1.0,-0.3) -- (-1.0,1.7) node [black,midway,xshift=-0.8cm] {$G_{1}$};
	\draw [decorate,decoration={brace,amplitude=6pt},xshift=-4pt,yshift=0pt] (-1.0,1.7) -- (-1.0,3.6) node [black,midway,xshift=-0.8cm] {$G_{2}$};

	\draw (0.1,0) -- (2.9,0) node[pos = 0.65, anchor=north ] { $ \sim \epsilon_{ij}$};
	\draw (0.1,0) -- (2.9,2) node[pos = 0.65, anchor=north,  sloped ] {$\sim X_{ij}$};
	\end{tikzpicture}
	\caption{Calculating the weight of a cut $\tilde{G}_{1}$, $\tilde{G}_{2}$. The weight of an edge between vertex $i$ in $G_{1,1}$ and vertex $j$ in $G_{2,1}$ is the random variable $\epsilon_{ij}$. The weight of an edge between vertex $i$ in   $G_{1,1}$ and vertex $j$ in $G_{2,2}$ is the random variable $X_{ij}$.}
	\label{fig:prob_plot}
\end{figure}
		
The weight of a cut $\{ \tilde{G}_{1}, \tilde{G}_{2} \}$ corresponding to the parameters $p_{1}$ and $p_{2}$ can be expressed in terms of $X_{ij}$ and $\epsilon_{ij}$ as follows.  It equals to the sum of weights of edges between $G_{1,1}$ and $G_{2,1}$ (with weight $\epsilon_{ij}$), between  $G_{1,1}$ and $G_{2,2}$ (with weight $X_{ij}$), between $G_{1,2}$ and $G_{2,1}$ (with weight $X_{ij}$), and between $G_{1,2}$ and $G_{2,2}$ (with weight $\epsilon_{ij}$). A graphical illustration of this setup is given in Figure~\ref{fig:prob_plot}. Formally, by denoting $S(p_{1},p_{2}) = W(\tilde{G}_{1},\tilde{G}_{2})$ (see~\eqref{eq:cut weight}), we have that
\begin{align}
	S(p_1,p_2) &= \sum_{i \in G_{1,1}}\left(\sum_{j \in G_{2,1}}\epsilon_{ij} + \sum_{j \in G_{2,2}} X_{ij}\right)  + \sum_{i \in G_{1,2}}\left( \sum_{j \in G_{2,2}} \epsilon_{ij} + \sum_{j \in G_{2,1}} X_{ij}\right) \notag \\
	&=	\sum_{(i,j) \in B_{X}} X_{ij} + \sum_{(i,j) \in B_{\epsilon}}\epsilon_{ij} \label{eq:split S} \\
&\leq \sum_{(i,j) \in B_{X}} (X'_{ij}+2\varepsilon) + |B_{\epsilon}| 4\varepsilon, \label{eq:cut_score}
\end{align}
where
\begin{equation}\label{eq:B_def}
B_{X} = \left ( G_{1,1} \times G_{2,2} \right ) \cup \left ( G_{1,2} \times G_{2,1} \right ), \quad B_{\epsilon} = \left ( G_{1,1} \times G_{2,1} \right )\cup \left ( G_{1,2} \times G_{2,2}\right),
\end{equation}
and~\eqref{eq:cut_score} was derived using~\eqref{eq:eps_bound} and~\eqref{eq:X_bound_X'}.
Note that
\begin{equation}\label{eq:B_sizes}
|B_{X}| = M^2 (p_1 + p_2 - 2p_1p_2), \quad |B_{\epsilon}| =  M^2 (p_1(1-p_1) + p_2(1-p_2)).
\end{equation}
By~\eqref{eq:split S} and~\eqref{eq:X_bound_X'}, the value of $S(p_1,p_2)$ for $p_1 = 0$ and $p_2 = 1 $ (or symmetrically for $p_1 = 1$ and $p_2 = 0 $) satisfies
\begin{equation}\label{eq:score_01}
	S(0,1) \geq \left[\sum_{	(i,j) \in B_{X}} X'_{ij}  \right]- 2\varepsilon M^2,
\end{equation}
where we have used the fact that in this case $\abs{B_{X}}= M^2$ and $\abs{B_{\epsilon}}=0$. The expected value of the right hand side of~\eqref{eq:score_01} is $M^2 \E(X') - 2\varepsilon M^2$, where $X'$ denotes any of the i.i.d random variables $X'_{ij}$, and $ \E $ denotes the expected value of a random variable. By Chebyshev's inequality,
\begin{equation}\label{eq:score_good}
P\left(S(0,1) < M^2 \E(X') - 2\varepsilon M^2 - kM \sigma(X')\right) < \frac{1}{k^2}.
\end{equation}
Choosing $k = \sqrt{M}$ we have
\begin{equation}\label{eq:S01 lower bound}
S(0,1) \ge M^2 \E(X') - 2\varepsilon M^2 - M^{3/2} \sigma(X'),
\end{equation}
with probability that converges to $1$ as $M\to \infty$. Note that the choice $k = \sqrt{M}$ is rather arbitrary, as for our purpose $ k $ can be any function of $ M $ as long as $ k /M \to 0$ as $ k\to \infty $.
	
Next, we denote by $U\subset [0,1]^2$ the set of pairs $(p_1,p_2)$ such that $p_{1}M$ and $p_{2}M$ are integers, and for any $0 \leq \delta \leq 1$, we denote by $U_\delta$ the subset of $U$ such that $p_{1}\geq \delta, 1-p_{2}\geq \delta  \mbox{  or  }  1- p_1\geq\delta, p_2\geq\delta $ (each $(p_{1},p_{2}) \in U_{\delta}$ corresponds to a partition that is $\delta$-precise in~\eqref{eq:Gkl_details}).

Since the Goemans-Williamson algorithm returns with high probability a cut whose score is at least 0.87 of the maximal cut, it returns (with high probability) a score higher than $ 0.87 $ times the score of $ S(0,1) $. Thus, if we show that the maximal score of a non-$\delta$-precise partition is less than 0.87 times the correct partition ($ S(0,1) $), then such a partition cannot be returned. 
In other words, with high probability Algorithm~\ref{alg:theoretical} returns a $\delta$-precise partition. An illustration of this argument is given in Figure \ref{fig:ilustration_proof_logic}.
\begin{figure}
	\begin{equation*}
	\left.\begin{array}{c}
	\mbox{GW returned score} \geq 0.87 \times \mbox{maximal score} \geq 0.87 S(0,1) \\
	\\
	
	\mbox{We will show that: } \\0.87 S(0,1) > \max(\mbox{score of non-}\delta\mbox{ precise partition})
	\end{array}
	\right\}
	\Rightarrow
	\mbox{GW return-}\delta\mbox{ precise partition}
	\end{equation*}
	\caption{Illustration of the idea of the proof}
	\label{fig:ilustration_proof_logic}
\end{figure}
Formally, we need to show that
\begin{equation} \label{eq:sucsess_ineq}
P \left( 0.87 S(0,1) < \max\limits_{(p_1,p_2) \in U \backslash U_\delta} S(p_1,p_2) \right) \ll 1.
\end{equation}
Thus, we will focus on finding the maximal $\delta$ such that with high probability
\begin{equation}\label{eq:delta condition}
0.87 S(0,1) \ge \max\limits_{(p_1,p_2) \in U \backslash U_\delta} S(p_1,p_2).
\end{equation}
	
We next estimate the maximum score over all cuts with $(p_{1},p_{2}) \in U \backslash U_\delta$. Since there are $2^{2M}$ possible cuts $(\tilde{G}_{1}, \tilde{G}_{2})$, we have that $|U \backslash U_\delta| \leq |U| = 2^{2M}$ (note that we assumed that $|G_1|= |G_2| = M$, but Algorithm~\ref{alg:theoretical} can return $\tilde{G}_1$ and $\tilde{G}_2$ of any size). Using~\eqref{eq:cut_score} and Lemma~\ref{lem:max2n_bounded} in Appendix~\ref{sec:distributions} (with $ a=2 $) we get that 
\begin{equation}\label{eq:score_of_bad}
\begin{array}{lll}
\max\limits_{(p_1,p_2) \in U \backslash U_\delta} S(p_1,p_2) & < &\max\limits_{(p_1,p_2) \in U \backslash U_\delta}  \sum\limits_{(i,j) \in B_{X}} (X'_{ij}+2\varepsilon) + |B_{\epsilon}| 4\varepsilon \\ &< &\max\limits_{(p_1,p_2) \in U \backslash U_\delta}\Big\{M^2\left[ (p_1 + p_2 - 2p_1p_2)(\E(X') + 2\varepsilon) + (p_1(1-p_1) + p_2(1-p_2)) 4\varepsilon\right] \\ & &
+2\sqrt{\frac{\log(|U \backslash U_\delta|)}{2}} \sqrt{M^2(p_1 + p_2 - 2p_1p_2)} \Big\}\\
&< & \max\limits_{(p_1,p_2) \in U \backslash U_\delta} \Big\{M^2\left[ (p_1 + p_2 - 2p_1p_2)(\E(X') + 2\varepsilon) + (p_1(1-p_1) + p_2(1-p_2)) 4\varepsilon\right] \\ & &
+M\sqrt{4M} \sqrt{(p_1 + p_2 - 2p_1p_2)}\Big\}
\end{array}
\end{equation}
with probability that converges to $1$ as $M\to \infty$.
If we now require that
\begin{multline} \label{eq:max_to_all_ineq_withM}
0.87 \left(M^2 \E(X') - 2\varepsilon M^2 -kM\sigma(X')\right) \geq \\ \max_{(p_1,p_2) \in U \backslash U_\delta}
M^2 [ (p_1 + p_2 - 2p_1p_2)(\E(X') + 2\varepsilon) + (p_1(1-p_1) + p_2(1-p_2)) 4\varepsilon ] \\+M\sqrt{4M} \sqrt{(p_1 + p_2 - 2p_1p_2)},
\end{multline}
then, by~\eqref{eq:S01 lower bound} and~\eqref{eq:score_of_bad} we get that~\eqref{eq:delta condition} holds, and so does~\eqref{eq:sucsess_ineq}, as required.
Additionally, for a large enough $M$, we neglect all terms that are not $O(M^2)$ in~\eqref{eq:max_to_all_ineq_withM}, and thus we require
\begin{multline} \label{eq:max_to_all_ineq}
0.87 \left(M^2 \E(X') - 2\varepsilon M^2 \right) \geq \\ \max_{(p_1,p_2) \in U \backslash U_\delta}
M^2 [ (p_1 + p_2 - 2p_1p_2)(\E(X') + 2\varepsilon) + (p_1(1-p_1) + p_2(1-p_2)) 4\varepsilon ].
\end{multline}

To sum up the proof thus far, we showed above that if $\delta$ is such that $0.87 S(0,1) > \max\limits_{(p_1,p_2) \in U \backslash U_\delta} S(p_1,p_2)$ with probability (denoted henceforth by) $ \rho $ that converges to $1$ as $M \to \infty$, then with probability $ \rho $ Algorithm~\ref{alg:theoretical} returns a $\delta$-precise partition.
We also showed that if~\eqref{eq:max_to_all_ineq} holds, then $0.87 S(0,1) > \max\limits_{(p_1,p_2) \in U \backslash U_\delta} S(p_1,p_2) $ with probability $ \rho $. Thus, for any $\delta$ such that~\eqref{eq:max_to_all_ineq} holds, Algorithm~\ref{alg:theoretical} returns a $\delta$-precise partition with probability $ \rho $ that converges to $1$ as $M\to \infty$.
 To find such a $\delta$, we rewrite~\eqref{eq:max_to_all_ineq} as
\begin{equation*}
0.87 \E(X') \geq  \max_{(p_1,p_2) \in U \backslash U_\delta}
(p_1 + p_2 - 2p_1p_2)\E(X') + \varepsilon (6 (p_1 + p_2) - 4(p_1^2 + p_2^2 + p_1p_2) +1.74),
\end{equation*}
and since $(6 (p_1 + p_2) - 4(p_1^2 + p_2^2 + p_1p_2) +1.74)\leq 4.74$ 
, we get,
\begin{equation*}
0.87 \E(X') \geq  \max_{(p_1,p_2) \in U \backslash U_\delta}
(p_1 + p_2 - 2p_1p_2)\E(X') + 4.74\varepsilon,
\end{equation*}
or
\begin{equation*}
0.87 - \frac{ 4.74\varepsilon}{\E(X')} \geq  \max_{(p_1,p_2) \in U \backslash U_\delta} (p_1 + p_2 - 2p_1p_2).
\end{equation*}
Since $\E(X')$ is the expectancy of the distance between two uniformly distributed random points on the unit sphere, which is equal to $\frac{4}{3}$~\cite{solomon1978geometric}, we have,
\begin{equation} \label{eq:quality_inequality2}
		0.87 - 3.555\varepsilon \geq \max_{(p_1,p_2) \in U \backslash U_\delta} (p_1 + p_2 - 2p_1p_2).
\end{equation}
That is, for any $\delta$ such that~\eqref{eq:quality_inequality2} holds, Algorithm~\ref{alg:theoretical} returns a $\delta$-precise partition with probability $ \rho $.

It is easy to see that the set
\begin{equation*}
\mathcal{U_\delta} = [0,1]^2 \backslash \left\lbrace (p_1,p_2)\in [0,1]^2 \ | \  p_{1}\geq \delta, 1-p_{2}\geq \delta  \text{ or }  1- p_1\geq\delta, p_2\geq\delta \right\rbrace
\end{equation*}
satisfies that $U\backslash U_{\delta} \subset \mathcal{U_\delta}$, and that  the maximum of $p_1 + p_2 - 2p_1p_2$ in $ \mathcal{U_\delta}$ is $\delta$ and is achieved on the boundary of $ \mathcal{U_\delta}$. Hence, $\max_{(p_1,p_2) \in U \backslash U_\delta} (p_1 + p_2 - 2p_1p_2) \leq \delta$. Thus, for any $\delta$ such that $0.87 - 3.555\varepsilon \geq \delta  $ we have that~\eqref{eq:quality_inequality2} holds,
and so we have a $\delta$-precise partition, with probability $ \rho $ that converges to $1$ when $N \to \infty$.
The largest $\delta$ for which the latter condition holds is $\delta=0.87 - 3.555\varepsilon$, and so Algorithm~\ref{alg:theoretical} returns a partition which is at least $0.87 - 3.555\varepsilon$ precise.
\end{proof}

The proof of the more general case $\abs{G_{1}} \neq \abs{G_{2}}$ follows the same steps, but with different constants appearing in the derivation starting with equation~\eqref{eq:B_sizes} and on. The case where $K \neq 2$ would require to change the notation in~\eqref{eq:Gkl} and the score $S$ in~\eqref{eq:split S}, but otherwise the proof remains conceptually the same.

\section{Experimental Results}
\label{sec:experimental_results}
In this section, we show results of Algorithm~\ref{alg:rot_maxcut} for simulated data sets (Section~\ref{subsec:sim}), and compare its performance with RELION~\cite{scheres2012relion,scheres2016processing} (Section~\ref{subsec:compareRelion}). In both cases, the performance is measured by counting the correctly and incorrectly classified images.

In the experiments we quantify the level of noise in the images using the Signal to Noise Ratio (SNR), defined by
\begin{equation}\label{SNR_def}
\mbox{SNR} = \frac{\mbox{var}(\phi_{av})}{\mbox{var}(noise)},
\end{equation}
where $ \Phi_{av}$ is the average volume and var is the variance of a signal. Since the noise we add to the images is Gaussian with zero mean and standard deviation $\sigma$, the SNR~\eqref{SNR_def} becomes
\begin{equation*}
\mbox{SNR} = \frac{\mbox{Power}(\phi_{av})}{\sigma^2},
\end{equation*}
or, for any given SNR level, the noise added to the images Gaussian with standard deviation
\begin{equation*}
\sigma = \sqrt{\frac{\mbox{Power}(\phi_{av})}{\mbox{SNR}}}.
\end{equation*}

In~\cite{katsevich2013covariance,anden2017structural}, the authors define a ``heterogeneous SNR" ($ \mbox{SNR}_{het}$) differently than~\eqref{SNR_def}. 
For the particular case described in Section~\ref{subsec:sim}, $ \mbox{SNR}_{het} \approx 0.5 \, \mbox{SNR}$.
\subsection{Implementation Notes}\label{sec:implementation_notes}
While in the analysis of Algorithm~\ref{alg:rot_maxcut} we use the Goemans-Williamson algorithm for the max-cut problem, for which the best performance bound for a polynomial time algorithm can be proven, its results can be improved in practice by applying the  following simple heuristic: start with the output of the Goemans-Williamson algorithm as an initial guesses for the cut, make a local search around the guess, and return the best cut detected.  This heuristic is described in Algorithm~\ref{alg:flip}. This heuristic can be further improved by starting from multiple initial guesses, one of which is the output of the Goemans-Williamson algorithm. For simplicity of notation, Algorithm~\ref{alg:flip} is presented for the case of two classes. A similar algorithm for $K>2$ is easily deduced. In the experiments below, after running the Goemans-Williamson algorithm, we ran Algorithm~\ref{alg:flip} with three initializations - one is the output of the GW algorithm and the other two are random initializations. Since Algorithm~\ref{alg:flip} only improves the score, these initializations guarantee a final cut with the same bound as Goemans-Williamsons algorithm for $ K=2 $ (or the Frieze-Jerrum algorithm for $ K>2 $). We noticed that the  naive approach of Algorithm~\ref{alg:flip} improves in some cases the results.

All the experiments were executed in MATLAB on a computer with two Intel Xeon X5560 CPUs running at 2.8GHz and an nVidia GTX TITAN 1080 GPU. The GPU was used only for the common lines search.

\begin{algorithm}
	\caption{The ``flip" algorithm for the max-cut problem}
	\label{alg:flip}
	\begin{algorithmic}[1]
		\State {\bfseries Input:}\begin{tabular}[t]{ll}
			$W$ & weights matrix of size $N\times N$.\\
			$ v $ &  vector with values $ \pm 1 $ that represent an initial classification.
		\end{tabular}
		\State {\bfseries Output: }{$v = (v_1, \ldots, v_N)$ with $v_i = 1$  if node $i$ is assigned to class 1, and $v_i = -1$ otherwise.}
		\Statex \(\triangleright\) \underline{Initialization}\medskip
		\State $score \gets v^tWv$ \Comment{Equivalent to  $\sum_{i,j:v_i \neq v_j} W_{ij}$.}
		\Repeat
		\State $score\_improved \gets false$
		\For {$i=1$ to $N$}
		\State $v_i \gets -v_i$ \Comment{Flipping assignment of node $i$.}
		\If {$v^tWv > score$}
		\State $score \gets v^tWv$
		\State $score\_improved \gets true$
		\Else
		\State $v_i = -v_i$
		\EndIf
		\EndFor
		\Until{$score\_improved == false$}
		\State\Return $v$
	\end{algorithmic}
\end{algorithm}

\subsection{Simulated molecules ($K=2,3$)} \label{subsec:sim}

To evaluate the performance of Algorithm~\ref{alg:rot_maxcut}, we first applied it to simulated heterogeneous data sets at various levels of noise. 
The heterogeneous data sets for the experiment were generated as follows. First, we created two three-dimensional volumes (molecules) corresponding to two types of molecules. The first molecule was a known density map of the 50S subunit of the E. coli ribosome, and the second was its perturbed version created by adding a small ball. The two three-dimensional volumes were chosen deliberately to have similar structures. Visualizations of the three-dimensional volumes are given in Figure~\ref{fig:vol1vol2}. Then, we generated 2500 noiseless projection images of each of the volumes, using uniformly distributed random orientations. The resulting 5000 images, each of size $65 \times 65$ pixels, comprised our noiseless heterogeneous data set. Finally, for each level of noise, we added to each image in the noiseless data set additive white Gaussian noise at the given noise level, and applied our algorithm to the resulting noisy data set. Note that in this case the two classes have equal sizes with $N_{1}=N_{2}=2500$. As mentioned above,  we used for the max-cut step the Goemans-Williamson algorithm and then improved its output using Algorithm~\ref{alg:flip} with 3 initial starting points (see Section~\ref{sec:implementation_notes}).

The results of the experiments are summarized in Table~\ref{tab:results_2classes_balanced}. Each row in the table corresponds to an experiment at a fixed SNR, and shows the number of images assigned to each class, the percentage of correctly detected common lines (defined as common lines whose estimated locations deviate by up to $10^{\circ}$ from their true locations known from the simulation), and the precision of the partition measured according to Definition~\ref{def:good_class}. To illustrate the SNR values used, we show in Figure~\ref{fig:mole_with_noise} a clean image and its noisy realizations at different levels of noise. In Figure~\ref{fig:sym_2classes_balanced_results} we show the reconstructed volumes using the noisy images and the estimated rotations and class partitions.

\begin{figure}
	\centering
	\includegraphics[width=0.4\textwidth]{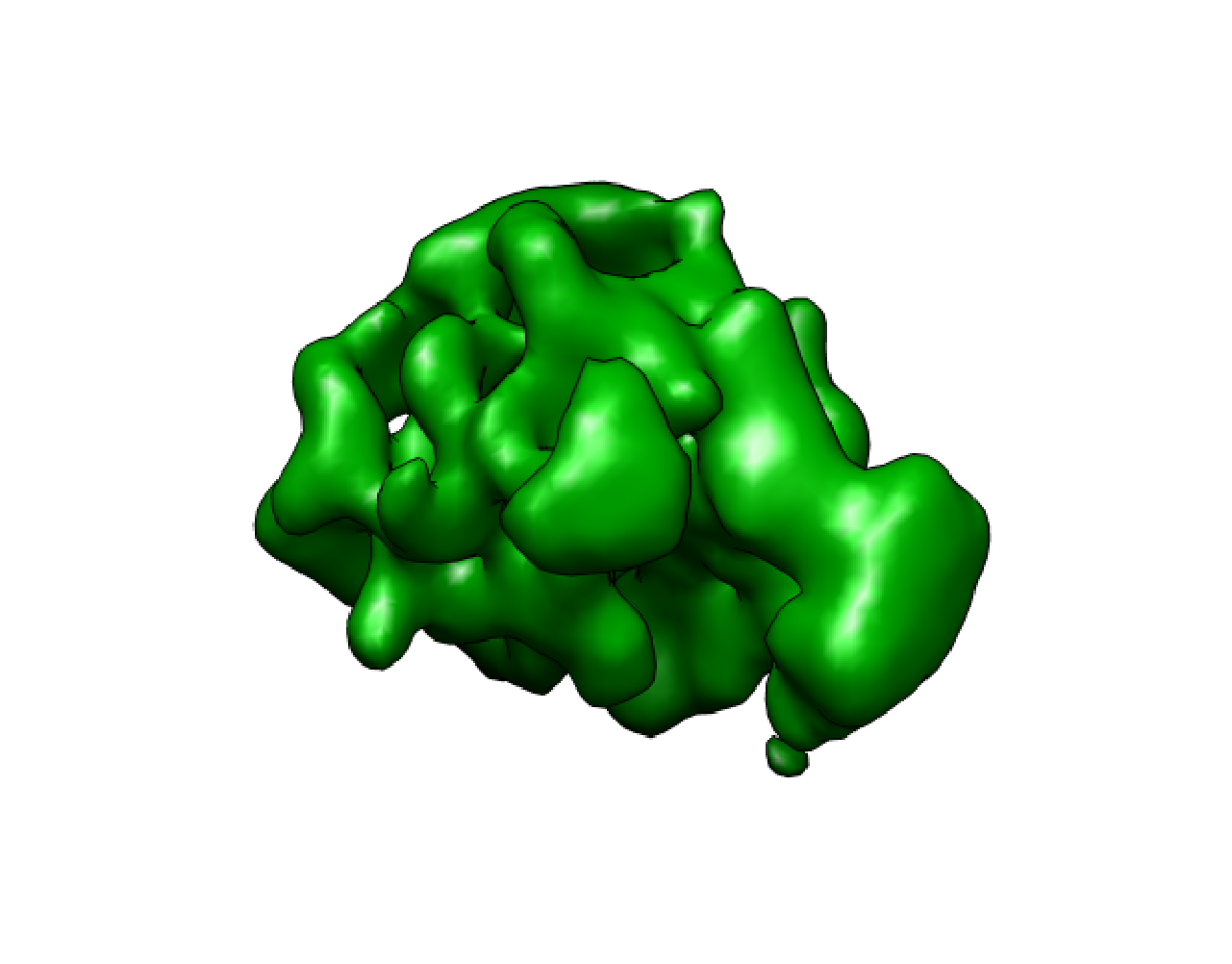}		\includegraphics[width=0.4\textwidth]{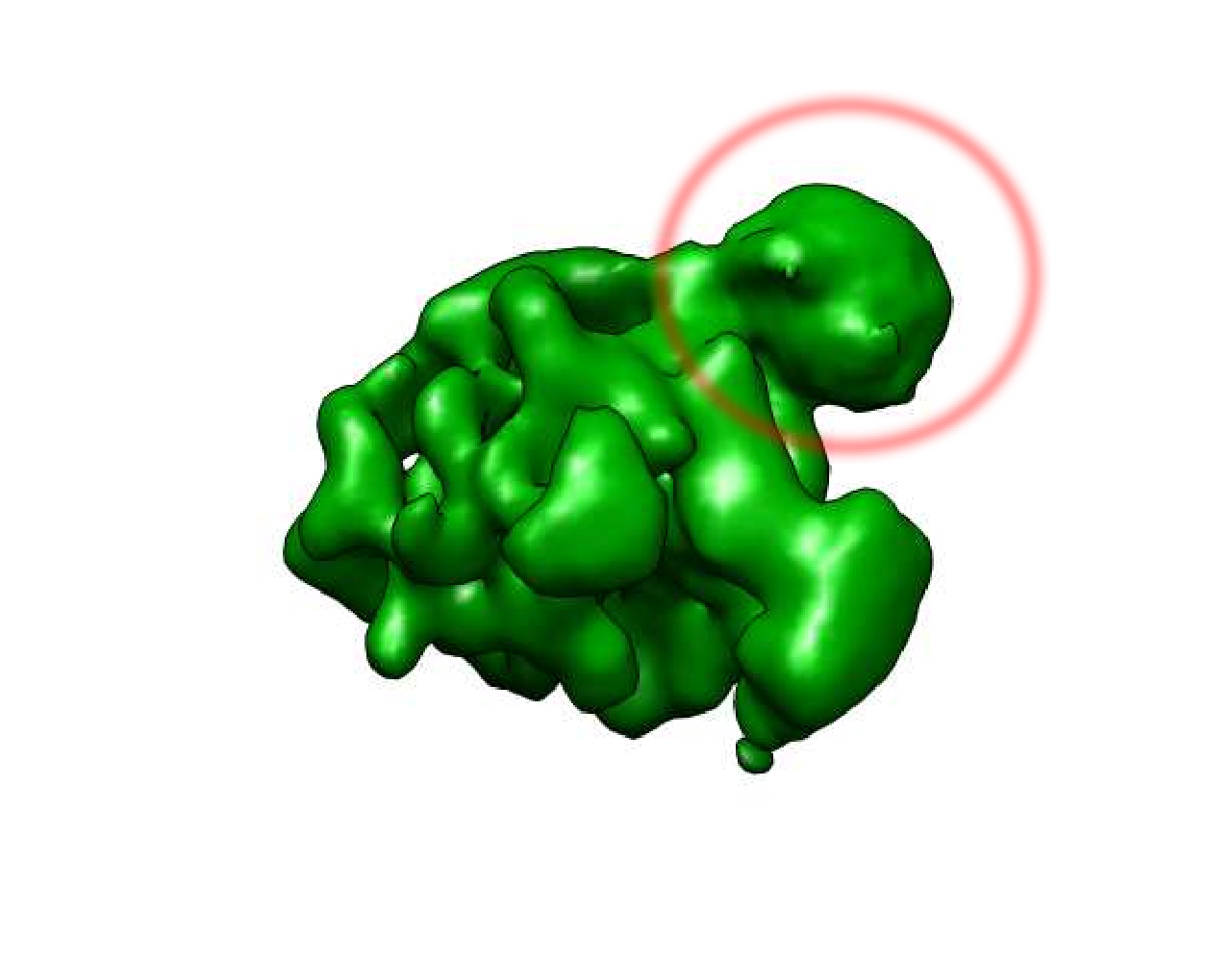}
	\caption{Volumes used in the $K=2$ simulated data experiment.}
	\label{fig:vol1vol2}
\end{figure}

\begin{figure}
	\centering
	\subfloat[Without noise]{
		\includegraphics[width=0.15\textwidth]{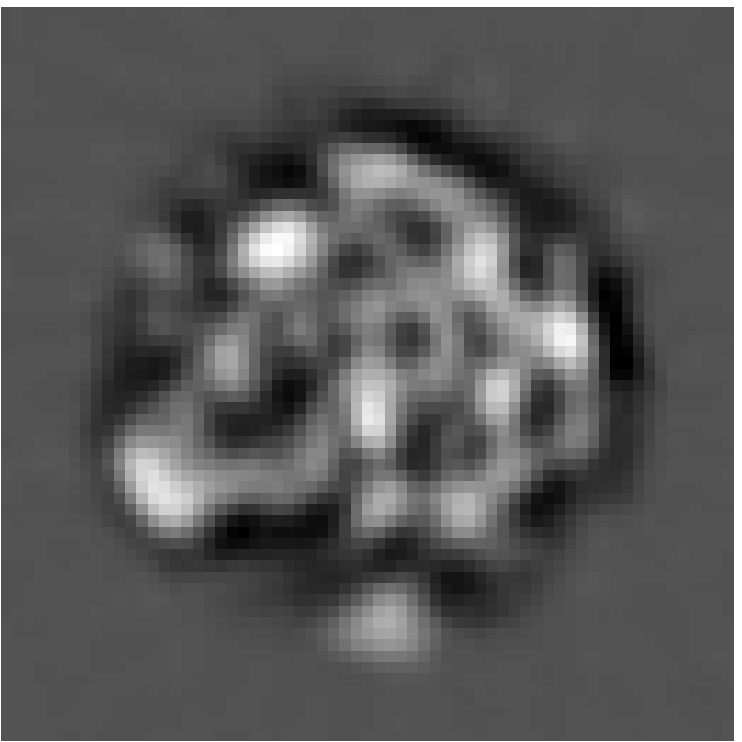}
		\label{fig:mol1_2d_SNR_inf}
    }
	\subfloat[SNR = 1]{
		\includegraphics[width=0.15\textwidth]{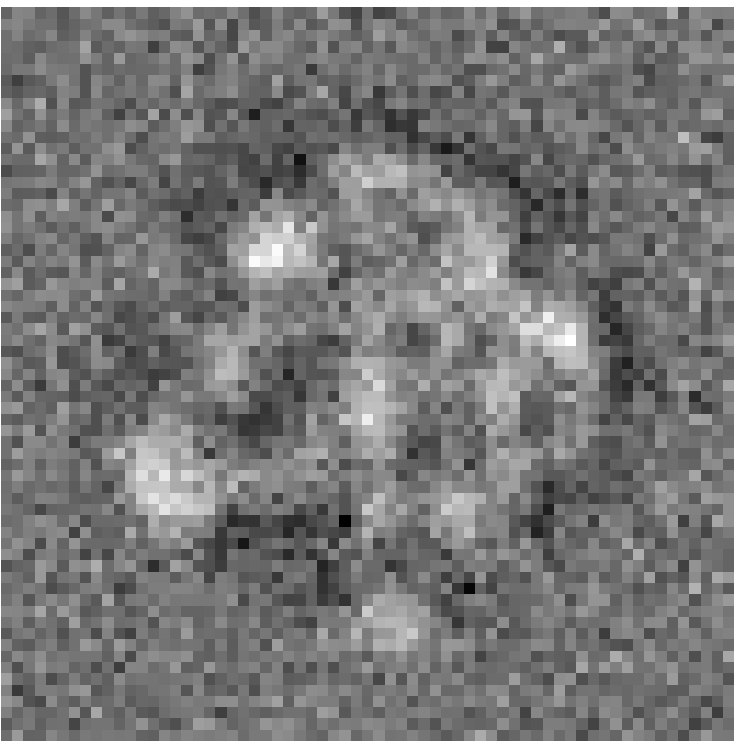}
		\label{fig:mol1_2d_SNR100}
	}
	\subfloat[SNR = 0.5]{
		\includegraphics[width=0.15\textwidth]{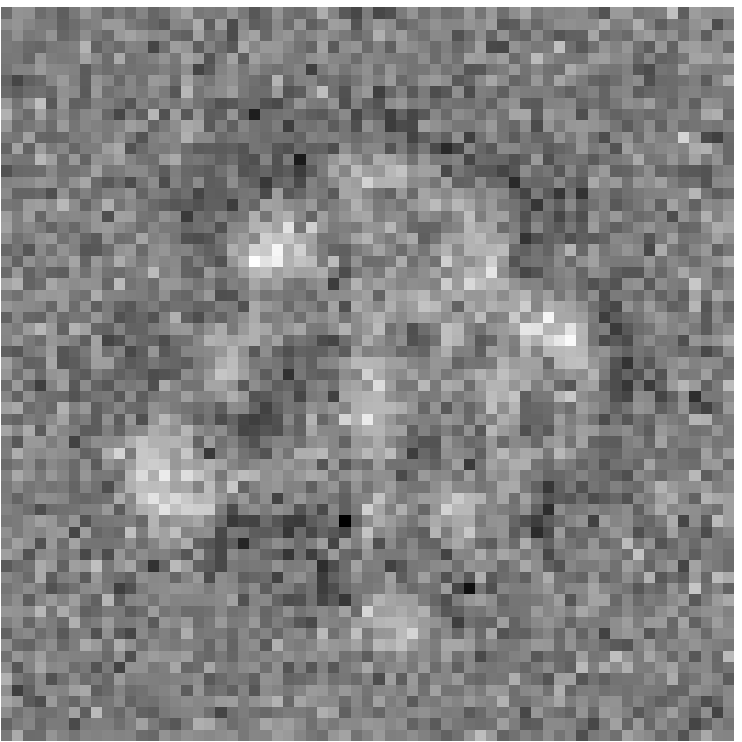}
		\label{fig:mol1_2d_SNR050}
	}
	\subfloat[SNR = 0.15]{
		\includegraphics[width=0.15\textwidth]{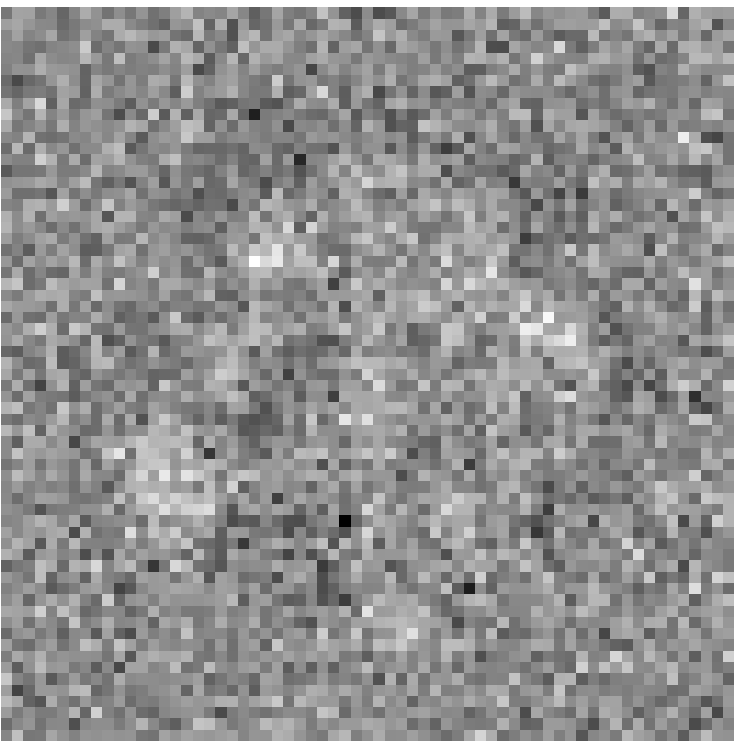}
		\label{fig:mol1_2d_SNR015}
	}
	\subfloat[SNR = 0.1]{
		\includegraphics[width=0.15\textwidth]{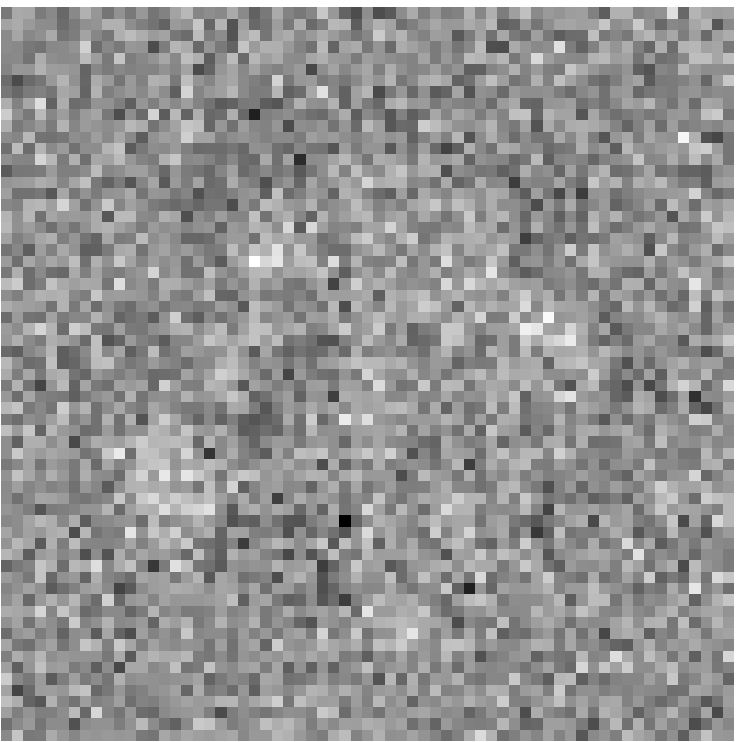}
		\label{fig:mol1_2d_SNR010}
	}
	\subfloat[SNR = 0.05]{
		\includegraphics[width=0.15\textwidth]{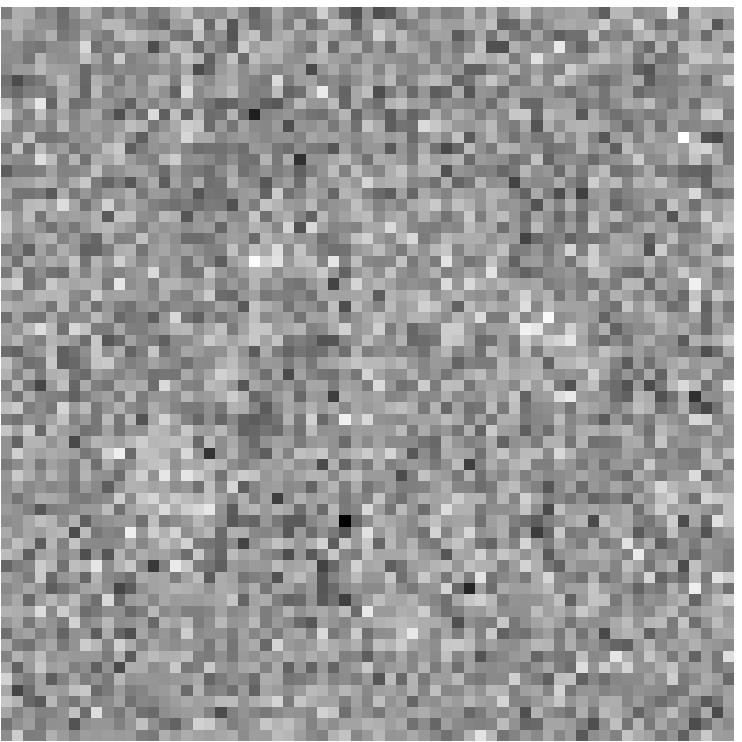}
		\label{fig:mol1_2d_SNR005}
	}
	
	\caption{A clean projection image and its noisy realizations at different levels of noise. Each image is of size $65 \times 65$ pixels.}
	\label{fig:mole_with_noise}
\end{figure}
\renewcommand{\arraystretch}{1.2}
\begin{table}
\begin{center}
		\begin{tabular}{|c|c|cc|c|c|c|}
			\hline
			\multirow{2}{*}{SNR} & \multirow{2}{55pt}{\centering{Estimated class}} & \multicolumn{2}{c|}{Correct class} & \multirow{2}{*}{Class size} & \multirow{2}{*}{Precision} &  \multirow{2}{80pt}{\centering{\% correct common lines}} \\ \cline{3-4}
							     & & class 1 & class 2 & & & \\ \hline
		
			\multirow{2}{*}{1} & class 1 & 2500 & 0 & 2500 &\multirow{2}{*}{1}& \multirow{2}{*}{91.48\%}\\
							   & class 2 & 0 & 2500 & 2500 & &\\
			\hline
			\multirow{2}{*}{0.5} & class 1 & 2500 & 0 & 2500 &\multirow{2}{*}{1}& \multirow{2}{*}{76.15\%}\\
								 & class 2 & 0 & 2500 & 2500 & &\\
			\hline
			\multirow{2}{*}{0.15} & class 1 &  2397 & 12    & 2409 &\multirow{2}{*}{0.960}& \multirow{2}{*}{38.07\%}\\
								  & class 2 & 103    & 2488 & 2591 & &\\
			\hline
			\multirow{2}{*}{0.1} & class 1 & $2137$ & $253$ & $2390$ &\multirow{2}{*}{0.861}& \multirow{2}{*}{23.46\%}\\
								 & class 2 &  $363$ & $2247$ & $2610$ & &\\
			\hline
			\multirow{2}{*}{0.05} & class 1 & $2084$ & $277$  & $2361$ &\multirow{2}{*}{0.846}& \multirow{2}{*}{ 15.86\%}\\
								  & class 2 & $416$  & $2232$ & $2639$ & &\\
			\hline
			\multirow{2}{*}{0.02} & class 1 & 1555  & 848  & 2403 &\multirow{2}{*}{0.636}& \multirow{2}{*}{ 3.11\%}\\
								  & class 2 & 945  & 1652 & 2597  & & \\
			\hline
			& & 2500 & 2500 & & &\\
			\hline
		\end{tabular}
		\caption{Results of Algorithm~\ref{alg:rot_maxcut} for balanced classes. For each SNR, the table presents a confusion matrix that shows how many images were assigned correctly (e.g. images belonging to class 1 that were assigned to class 1) and incorrectly (images belonging to class 1 that were assigned to class 2 or vice versa). The ``Class size" column shows the sizes of the estimated class 1 and 2 for each SNR. The precision column is calculated following Definition~\ref{def:good_class}, and the ``\% correct common lines'' is the percentage of the common lines detected within $10^{\circ}$ of their correct location. }
		\label{tab:results_2classes_balanced}
\end{center}	
\end{table}

\begin{figure}
	\begin{center}
		\begin{tabular}{l >{\centering\arraybackslash}m{0.15\textwidth}>{\centering\arraybackslash}m{0.15\textwidth}>{\centering\arraybackslash}m{0.15\textwidth}>{\centering\arraybackslash}m{0.15\textwidth}>{\centering\arraybackslash}m{0.15\textwidth}}
			& Original & SNR=0.15 & SNR=0.1 & SNR=0.05 & SNR = 0.02\\
			Class 1 & \includegraphics[width=0.2\textwidth]{Figures/Vol1_1_000108.eps} & \includegraphics[width=0.2\textwidth]{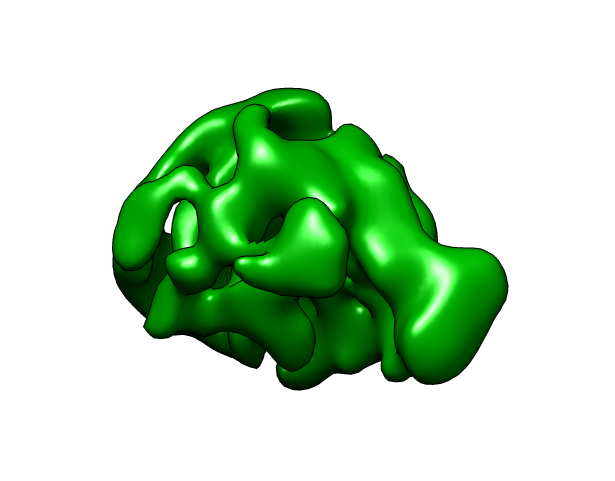} &
			\includegraphics[width=0.2\textwidth]{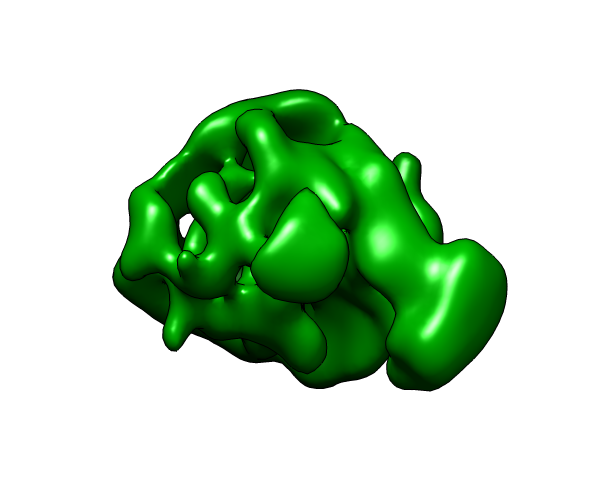} &
			\includegraphics[width=0.2\textwidth]{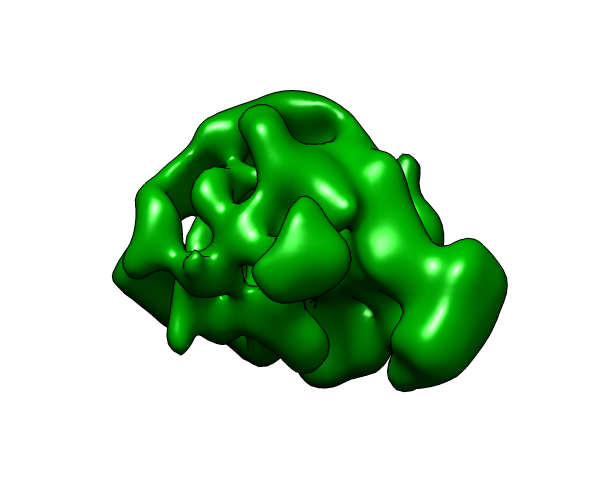} &
			\includegraphics[width=0.2\textwidth]{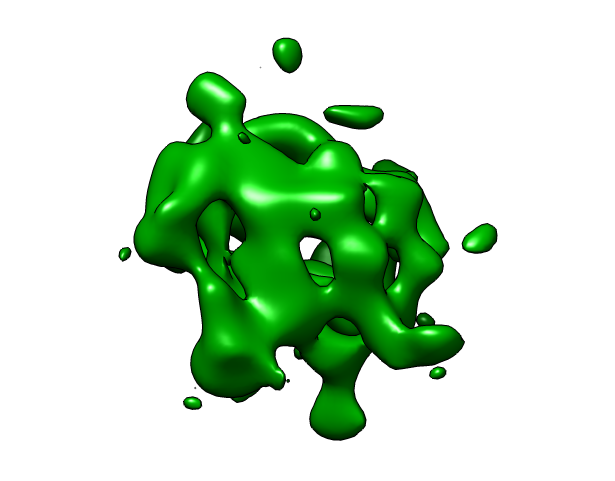}\\
			Class 2 & \includegraphics[width=0.2\textwidth]{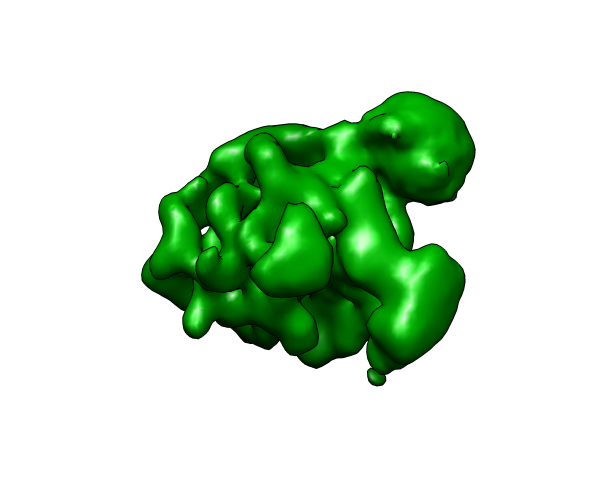} & \includegraphics[width=0.2\textwidth]{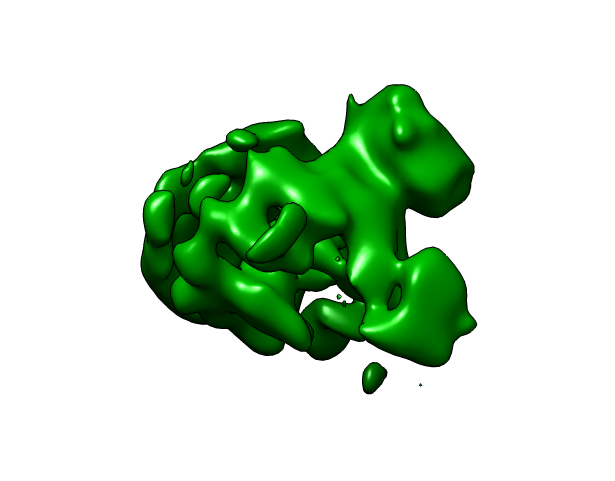} & \includegraphics[width=0.2\textwidth]{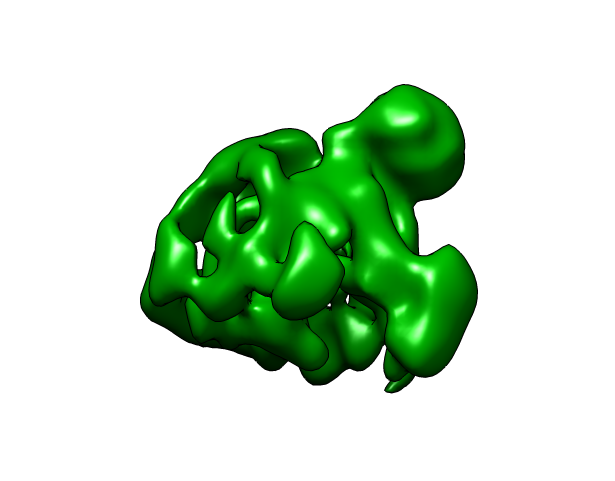} &
			\includegraphics[width=0.2\textwidth]{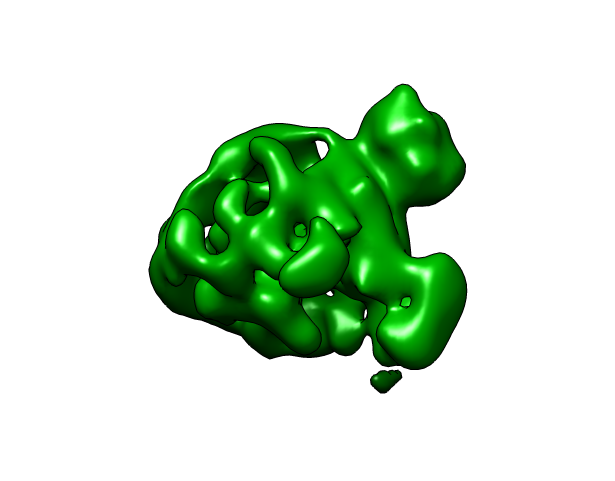} &
			\includegraphics[width=0.2\textwidth]{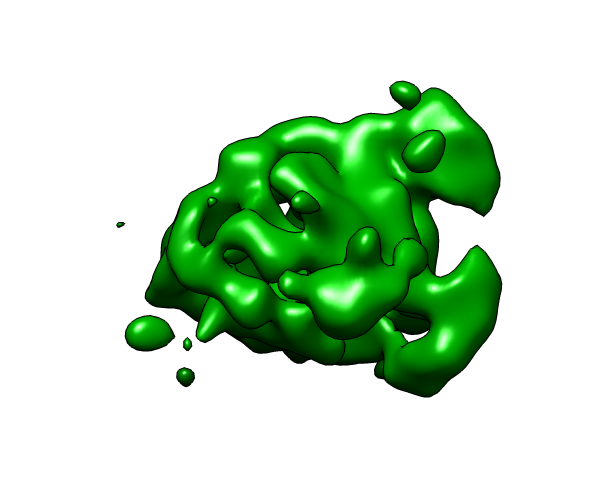}
		\end{tabular}
		\caption{The two volumes reconstructed from the heterogeneous data set at different levels of noise. Note that the volumes were reconstructed directly from the noisy images. Even at SNR = 0.05 it is easy to identify the additional ``ball" of Class 2 that does not appear in Class 1.}
		\label{fig:sym_2classes_balanced_results}
	\end{center}
\end{figure}

In the next experiment, we used the same setup as above, except that this time the classes were unbalanced, with $N_{1} = 4000$ and $N_{2} = 1000$. The results of this experiment are summarized in Table~\ref{tab:results_2classes_unbalanced}.
One can see from Tables~\ref{tab:results_2classes_balanced} and~\ref{tab:results_2classes_unbalanced} that the two classes returned by Algorithm~\ref{alg:rot_maxcut} tend to be of similar size. We see in Table~\ref{tab:results_2classes_unbalanced} that the estimated class 2 is always larger than its true size, and for higher noise levels it is very clear that the estimated classes tend to be of similar sizes. The intuition for this behavior is that for classes of similar size, the number of edges in the cut is maximal.
This behavior can be alleviated by reconstructing a volume from the class to which more images were assigned, which results in a more accurate model for one of the volumes. Using this volume, we can get more accurate estimates for the rotations and common lines, which in turn will result in a better estimate of the partition.

The running time of Algorithm~\ref{alg:rot_maxcut} depends on its initialization and on the running times of the LUD algorithm, the max-cut algorithm, and the number of iterations used.
In the experiments presented in this subsection, 
the total running time for the whole algorithm (including common lines search) for 5000 images was less then 12 hours allocated as follows: about 2 hours for the common lines search, 4 hours for the initial rotations assignment (this initial assignment is much slower then subsequent assignments, since all the images are in a single class), and about 2 more hours for each max-cut and rotations assignment steps.

\begin{table}
\begin{center}
		\begin{tabular}{|c|c|cc|c|c|c|c|}
			\hline
			\multirow{2}{*}{SNR} & \multirow{2}{55pt}{\centering{Estimated class}} & \multicolumn{2}{c|}{Correct class} & \multirow{2}{*}{Class size} & \multirow{2}{70pt}{\centering Precision of class 1}& \multirow{2}{45pt}{\centering Precision}& \multirow{2}{80pt}{\centering{\% correct common lines}} \\ \cline{3-4}
			& & class 1 & class 2 & & &&\\ \hline
			
			\multirow{2}{*}{1} & class 1 & 3974 & 0 & 3974 & \multirow{2}{*}{1} & \multirow{2}{*}{0.98} & \multirow{2}{*}{90.48\%}\\
							   & class 2 & 26 & 1000 & 1026 & & &\\
			\hline
			\multirow{2}{*}{0.5} & class 1 & 3479 & 0 & 3479 & \multirow{2}{*}{1}& \multirow{2}{*}{0.66}& \multirow{2}{*}{73.71\%}\\
			& class 2 & 521 & 1000 & 1521 & & &\\
			\hline

			\multirow{2}{*}{0.15} & class 1 & 2699  & 9   & 2708  &  \multirow{2}{*}{0.99}& \multirow{2}{*}{0.43}&  \multirow{2}{*}{39.15\%}\\
			& class 2 &  1301   & 991 & 2292 & && \\
			\hline
			\multirow{2}{*}{0.1} & class 1 & 2598 & 34 & 2632 & \multirow{2}{*}{0.99}& \multirow{2}{*}{0.41}&  \multirow{2}{*}{29.40\%}\\
			& class 2 &  1402 & 966 & 2368 & & & \\
			\hline
			\multirow{2}{*}{0.05} & class 1 &  $2505$ & $61$ & $2566$  & \multirow{2}{*}{0.98}& \multirow{2}{*}{0.39} &  \multirow{2}{*}{ 16.65\%}\\
			& class 2 &   $1495$ & $939$ & $2434$ & &&\\
			\hline
			\multirow{2}{*}{0.02} & class 1 & 2305 & 152  & 2457 & \multirow{2}{*}{0.94}& \multirow{2}{*}{0.33} &  \multirow{2}{*}{4.79\%}\\
			& class 2 & 1695  & 848 & 2543 & &&\\
			\hline
			& & 4000 & 1000 & & &&\\
			\hline
		\end{tabular}
		\caption{Results of Algorithm~\ref{alg:rot_maxcut} for unbalanced classes. For a detailed description of the different columns see Table~\ref{tab:results_2classes_balanced}.}
		\label{tab:results_2classes_unbalanced}
\end{center}
\end{table}

Finally, we tested the algorithm using simulated data with $K=3$ classes. In this case, the three volumes were a density map of the 50S subunit of the E. coli ribosome (the same volume used above for the experiment with $K=2$) and its two modified versions. Visualizations of the three volumes are shown in Figure~\ref{fig:vol1vol2vol3}.
In this experiment we used  $N_{1} = N_{2} = N_{3}= 1500$ and $ \mbox{SNR} = 0.15 $. The results are summarized in Table~\ref{tab:results_3classes}. It is noticeable that for the same noise levels, the precision for $ K=3 $ is lower than for $ K=2 $. One possible reason for this is that the bound provided by the Frieze-Jerrum algorithm for $ K=3 $ is worse than the bound for $ K=2 $.

\begin{figure}
	\centering
	\subfloat[Volume 1]{
        \includegraphics[width=0.25\textwidth]{Figures/Vol1_1_000108.eps}
        \label{fig:vol_3_1}
    }
    \subfloat[Volume 2]{
        \includegraphics[width=0.25\textwidth]{Figures/Vol1_m_1_000108_circle.eps}
        \label{fig:vol_3_2}
    }
	\subfloat[Volume 3]{
        \includegraphics[width=0.25\textwidth]{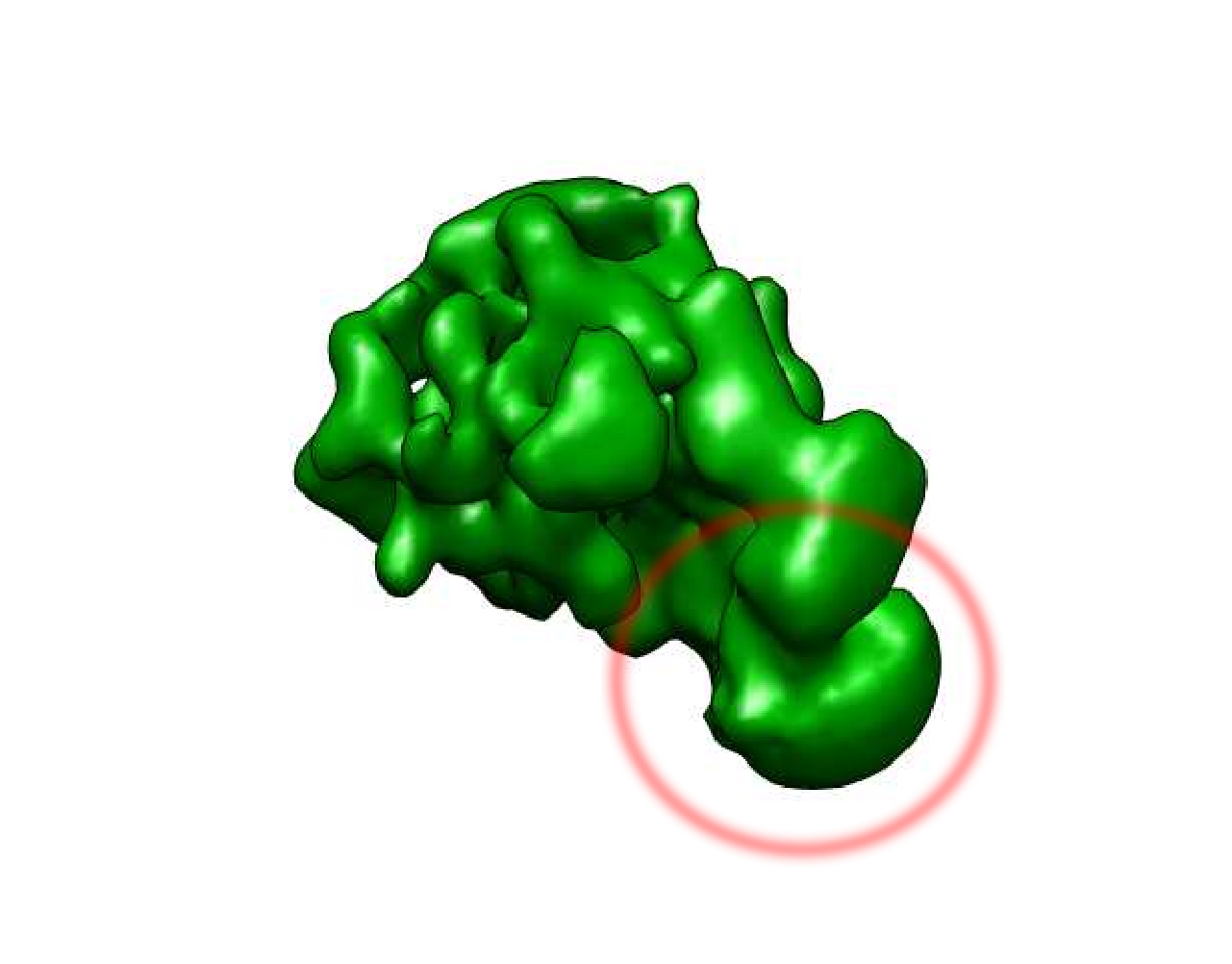}
        \label{fig:vol_3_3}
    }
	\caption{Volumes used in the $K=3$ simulated data experiment. \protect\subref{fig:vol_3_1} is the original volume; \protect\subref{fig:vol_3_2}~and~\protect\subref{fig:vol_3_3} are the modified versions of~\protect\subref{fig:vol_3_1}.}
	\label{fig:vol1vol2vol3}
\end{figure}

\begin{table}
\begin{center}
		\begin{tabular}{|c|c|ccc|c|c|c|}
			\hline
			\multirow{2}{*}{SNR} & & \multicolumn{3}{c|}{Correct class} & \multirow{2}{*}{Class size} & \multirow{2}{*}{Precision} & \multirow{2}{80pt}{\centering{\% correct common lines}} \\ \cline{3-5}
			& & class 1 & class 2 & class 3& & &\\ \hline
			
			\multirow{3}{*}{1.00} &
			  class 1 & 1452 &   0   &  0    & 1452 &\multirow{3}{*}{0.963}& \multirow{3}{*}{72.11\%}\\
			& class 2 &  35  &  1497 & 22    & 1554 &  &\\
			& class 3 & 13   &  3    &  1478 & 1494 &  &\\
					
			\hline
			\multirow{3}{*}{0.50} &
			class 1 & 1418 & 21   & 14   & 1513 &\multirow{3}{*}{0.935}& \multirow{3}{*}{63.61\%}\\
			& class 2 & 70   & 1435 & 29   & 1534 &  &\\
			& class 3 & 12   & 44   & 1457 & 1513 &  &\\
			
			\hline
			
			\multirow{3}{*}{0.15} &
			  class 1 & 1407 & 59   & 59   & 1525 &\multirow{3}{*}{0.888}& \multirow{3}{*}{38.14\%}\\
			& class 2 & 18   & 1345 & 85   & 1448 &  &\\
			& class 3 & 75   & 96   & 1356 & 1527 &  &\\
			\hline
			
			\multirow{3}{*}{0.10} &
 			  class 1 & 644 & 484 & 391 & 1519 &\multirow{3}{*}{0.33}& \multirow{3}{*}{12.15\%}\\
			& class 2 & 301 & 566 & 612 & 1479 &  &\\
			& class 3 & 555 & 450 & 497 & 1502 &  &\\
			\hline
			& & 1500 & 1500 & 1500 & & &\\
			\hline
		\end{tabular}
		\caption{Results of Algorithm~\ref{alg:rot_maxcut} for $K=3$ classes.}
		\label{tab:results_3classes}
\end{center}
\end{table}

\subsection{Comparison with RELION} \label{subsec:compareRelion}
We next compare the performance of Algorithm~\ref{alg:rot_maxcut} to that of RELION~\cite{scheres2012relion,scheres2016processing}, which 
implements expectation-maximization algorithms for image and volume classification, and for three-dimensional reconstruction. In particular, we demonstrate the well-known weakness of the expectation-maximization approach of providing a solution which is optimal only locally, by showing that in some cases RELION returns very accurate results (outperforming Algorithm \ref{alg:rot_maxcut}), while in other cases it fails even in ``simple" settings.

The first experiment uses simulated projections of yeast 80S ribosome-tRNA complexes. The two volumes used in this experiment are EMD-5976 and EMD-5977 from the Electron Microscopy Data Bank EMDB~\cite{EMDB}, corresponding to rotated and non-rotated conformations of the yeast 80S ribosome-tRNA complexes. According to the 0.5 threshold of the Fourier Shell Correlation (FSC), the two volumes agree to a resolution of is $ 19.27$~\AA. The FSC curve between EMD-5976 and EMD-5977 is presented in Figure \ref{fig:FSC_cleanVols_EMD5976}. We generated 3000 projections from each of the volumes, each projection of size $100\times 100$ pixels, with pixel size of $3.78$~\AA. White Gaussian noise was added to the images so that the resulting images have SNR of $0.5$. The images, corresponding STAR file, and the clean volumes are available at~\url{http://www.math.tau.ac.il/~aizeny/projects_cryo_hetero.html}.

\begin{figure}
	\centering
	\subfloat[FSC of EMD-5976 and EMD-5977 ]{
		\includegraphics[width=0.45\textwidth]{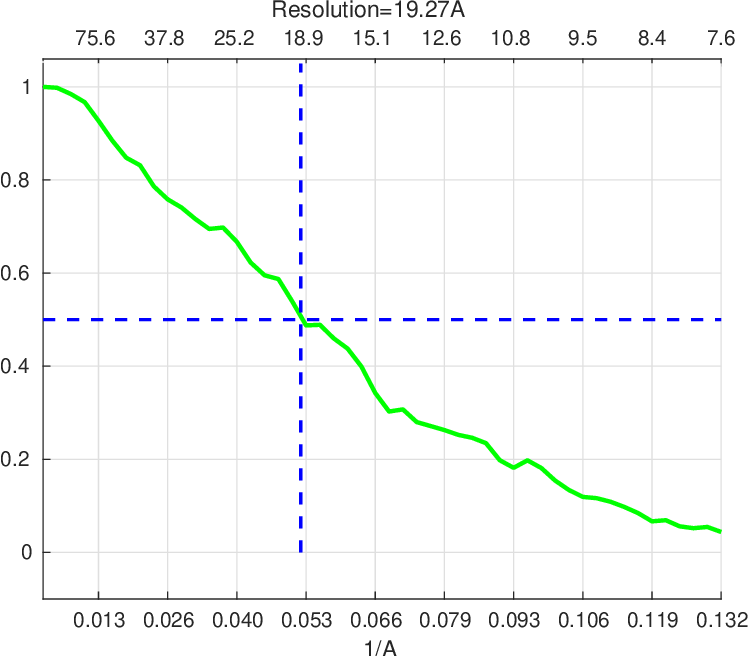}
		\label{fig:FSC_cleanVols_EMD5976}
	}~~
	\subfloat[FSC of EMD-0104 and EMD-0105]{
		\includegraphics[width=0.45\textwidth]{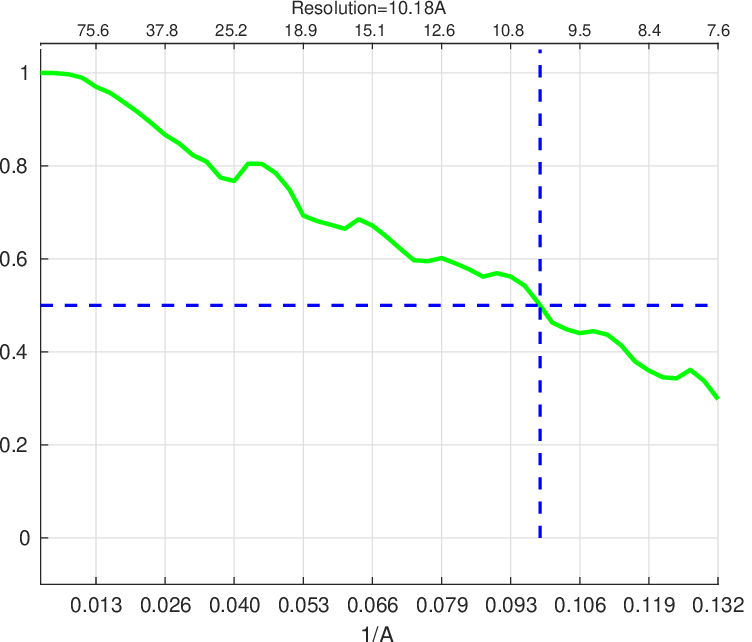}
		\label{fig:FSC_cleanVols_EMD0104}
	}
	\caption{FSC curves comparing the original volumes.}
	\label{fig:FSC_cleanVols_EMD5976_EMD0104}
\end{figure}

To apply Algorithm~\ref{alg:rot_maxcut} to this dataset, we initialized it with rotations generated using RELION's ``refine3D" procedure, where the ``Number of classes" parameter was set to~$1$. Given this initialization, Algorithm~\ref{alg:rot_maxcut} estimated a partition with precision $0.99$: in class~1 there were  $ 0 $ images of EMD-5976 and $ 2977$ images of EMD-5977, and in class~2 there were $ 3000 $ images of EMD-5976 and $ 23 $ images of EMD-5977. The weights matrix $W$ from Step~\ref{algstep:maxcut} of Algorithm~\ref{alg:theoretical} is shown in Figure~\ref{fig:W} (values shown in gray-scale). In Figure~\ref{fig:Wrp} the images are ordered in some random order (to emphasize the block structure of Figure~\ref{fig:Wa}), and in Figure~\ref{fig:Wa} the images are sorted such that the first 3000 images are of EMD-5976 and the next 3000 are of EMD-5977. Based on the classification of Algorithm \ref{alg:rot_maxcut}, we reconstructed two volumes using RELION. The two volumes were reconstructed with resolution of 8.5~\AA~and 8.6~\AA~compared to their corresponding volume from the EMDB, and with resolution of 19.8~\AA~and 20.0~\AA~compared to the other volume from the EMDB. The FSC curves are presented in Figure \ref{fig:FSC4_AlgRecon_EMD5976_vol12}. The total running time of Algorithm~\ref{alg:rot_maxcut} for this dataset on the platform described above is 6 hours and 40 minutes, out of which 6 hours and 20 minutes were required for finding the common lines between the images.

\begin{figure}
	\centering
	\subfloat[]{
		\includegraphics[width=0.45\textwidth]{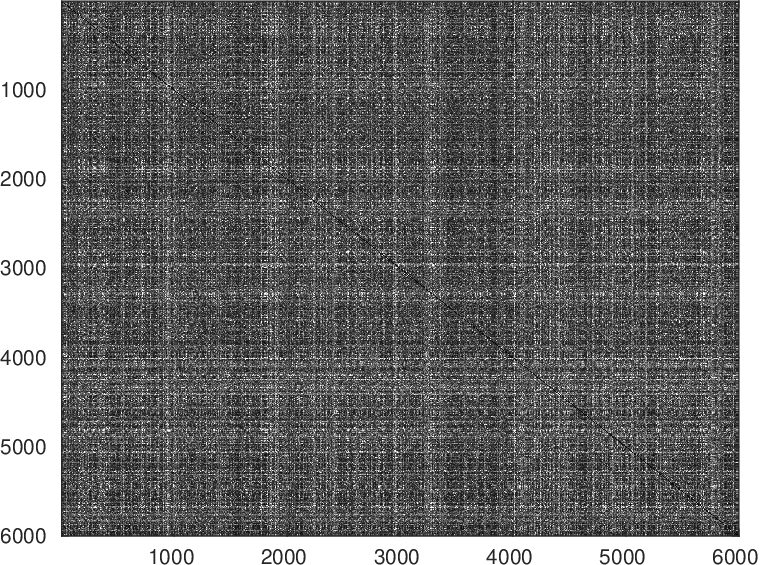}
		\label{fig:Wrp}
	}
	\subfloat[]{
		\includegraphics[width=0.45\textwidth]{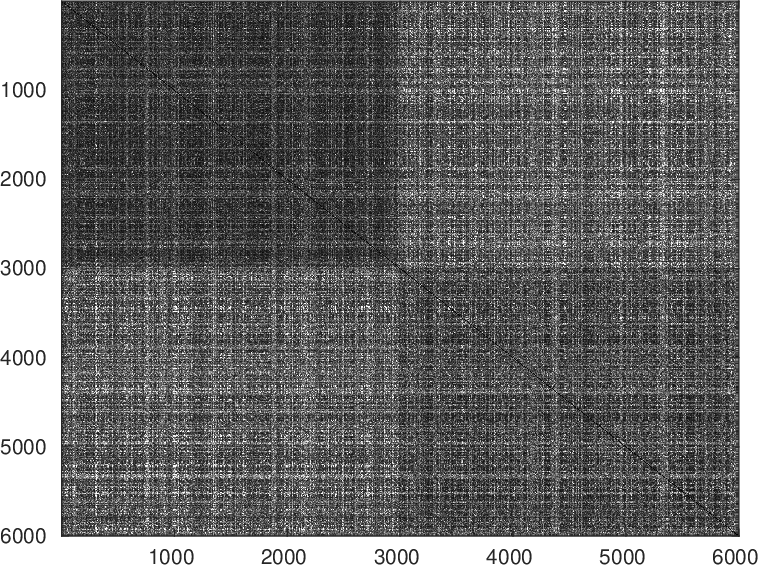}
		\label{fig:Wa}
	}	
	\caption{The weights matrix $W$ of Algorithm~\ref{alg:rot_maxcut}. In (a) The images are ordered randomly. In (b) images 1 - 3000 are from EMD-5977 and images 3001 - 6000 are from EMD-5976. Lighter shades correspond to larger weights.}
	\label{fig:W}
\end{figure}

\begin{figure}
	\centering
	\subfloat[FSC curves of volume 1 compared to EMD-5976 and EMD-5977. ]{
		\includegraphics[width=0.45\textwidth]{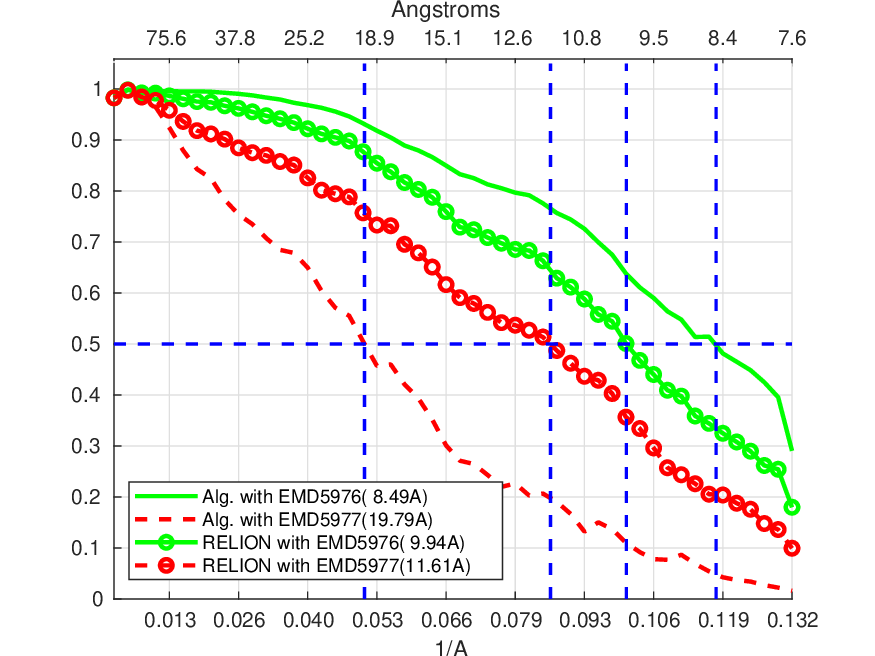}
		\label{fig:FSC4_AlgRecon_EMD5976_vol1}
	}~~
	\subfloat[FSC curves of volume 2 compared to EMD-5976 and EMD-5977.]{
		\includegraphics[width=0.45\textwidth]{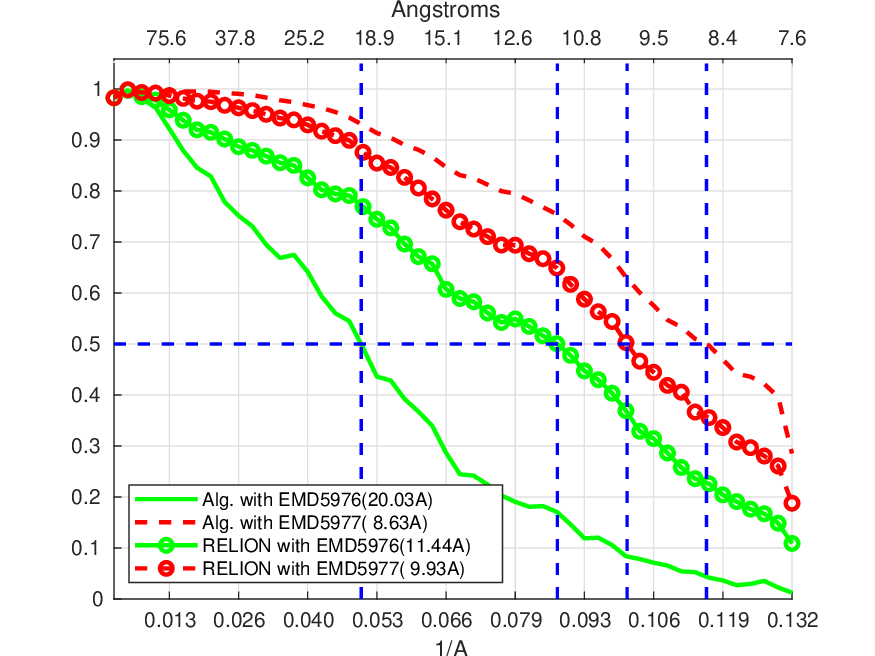}
		\label{fig:FSC4_AlgRecon_EMD5976_vol2}
	}
	\caption{FSC curves comparing the reconstructions made based on the partitions of Algorithm~\ref{alg:rot_maxcut} and RELION. The green (solid) lines are the FSC curves with EMD-5976, and the red (dashed) lines are the FSC curves with EMD-5977. Lines with markers are FSC curves of RELION based volumes, and lines without markers FSC curves of Algorithm \ref{alg:rot_maxcut} based volumes. }
	\label{fig:FSC4_AlgRecon_EMD5976_vol12}
\end{figure}

To initialize RELION's 3D classification algorithm for this dataset, we used the two ground truth volumes (used to simulate the projections). The classification estimated by RELION has precision of only $0.6$: in class~1 there were  1781 images of EMD-5976 and 1175 images of EMD-5977, and in class~2 there were 1219 images of EMD-5976 and 1825 images of EMD-5977.  During the tests we followed the common practice of running RELION with 25 classification iterations. 
Next we reconstructed the volumes corresponding to the two classes using RELION (based on the partition made by RELION). The two volumes were reconstructed with resolution of 9.9~\AA~and 9.9~\AA ~compared with their ground truth, and with resolutions of 11.6~\AA~and 11.4~\AA~compared with the ground truth of the other class.

In a second experiment, we used the same setting as above, but with EMD-0104 and EMD-0105. The resolution between the two volumes is $ 10.18 $~\AA~(so the volumes are more similar than the volumes in the previous experiment). The FSC curve between EMD-0104 and EMD-0105 is presented in Figure \ref{fig:FSC_cleanVols_EMD0104}. Algorithm~\ref{alg:rot_maxcut} estimated a partition whose precision is $0.7$. Note that the precision in this experiment is lower than the precision in the previous one (comparing EMD-5976 and EMD-5977), although we used the same number of images. Again, using the classification of Algorithm \ref{alg:rot_maxcut}, we reconstructed two volumes using RELION. The two volumes were reconstructed with resolution of 7.9~\AA~and 8.3~\AA~when compared with their ground truth, and with resolution of  9.5~\AA~and 9.6~\AA~when compared with the ground truth of the other class. The FSC curves are presented in Figure \ref{fig:FSC4_AlgRecon_EMD0104_vol12}. 
In this case, RELION returned a partition whose precision is $ 0.93 $. We do not have an explanation for the improvement in the performance of RELION between the two experiments. Next, we reconstructed the two volumes using RELION (based on the partition made by RELION). Each volume was reconstructed with resolution of 7.6~\AA~when compared with its ground truth, and with resolution of 10.6~\AA~and 11.0~\AA~when compared with the ground truth of the other class.

\begin{figure}
	\centering
	\subfloat[FSC curves of volume 1 compared to EMD-0104 and EMD-0105.]{
		\includegraphics[width=0.45\textwidth]{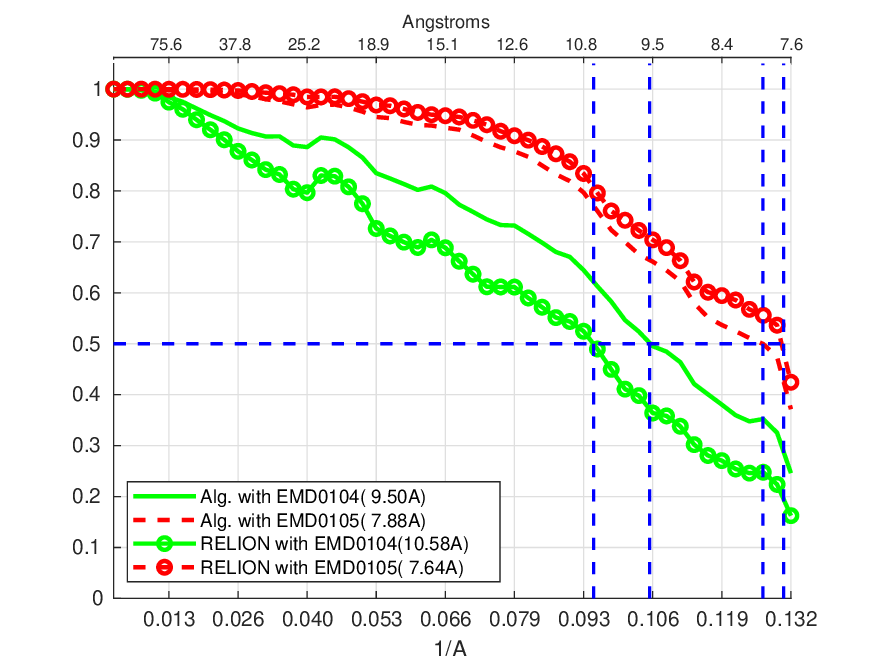}
		\label{fig:FSC4_AlgRecon_EMD0104_vol1}
	}~~
	\subfloat[FSC curves of volume 1 compared to EMD-0104 and EMD-0105.]{
		\includegraphics[width=0.45\textwidth]{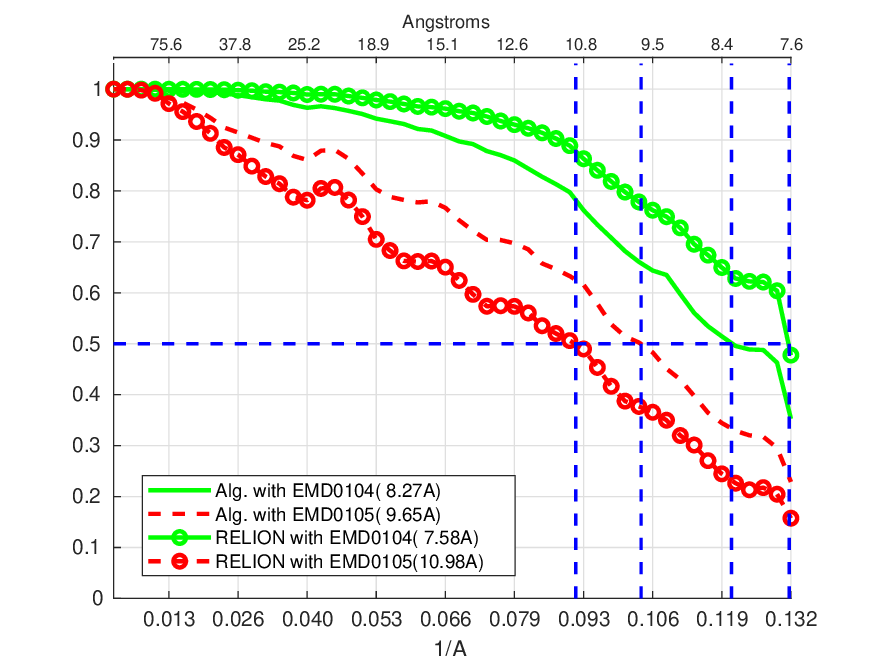}
		\label{fig:FSC4_AlgRecon_EMD0104_vol2}
	}
	\caption{FSC curves comparing the reconstructions made based on the partitions of Algorithm~\ref{alg:rot_maxcut} and RELION. The green (solid) lines are the FSC curves with EMD-0104, and the red (dashed) lines are the FSC curves with EMD-0105. Lines with markers are FSC curves of RELION based volumes, and lines without markers FSC curves Algorithm \ref{alg:rot_maxcut} based volumes.}
	\label{fig:FSC4_AlgRecon_EMD0104_vol12}
\end{figure}

Next, we show the dependency of the precision on the number of images in the two aforementioned experiments. We show in Tables \ref{tab:precision_diffN} and \ref{tab:precision_diffN_relion} a summary of the results of applying Algorithm \ref{alg:rot_maxcut} and RELION's classification in the same setting as above but with different numbers of projection images (the precision is shown for 4000, 6000, and 10000 projection images). It is noticeable that when the two structures in the heterogeneous dataset are ``more different" (like in the experiment of EMD-5976 and EMD-5977) then a relatively small number of projection images suffices for a ``good" partitioning by Algorithm \ref{alg:rot_maxcut}, and adding more projection images increases the precision. When the two structures in the heterogeneous dataset are similar (like in the experiment of EMD-0104 and EMD-0105), more projection images are necessary to get a high precision, and increasing the number of projection images again improves the precision.

In Table \ref{tab:precision_diffN_relion} we see inconsistency in the precision of the partition of RELION. For 6000 images the precision is worse than for 4000 images. It is also evident from Tables~\ref{tab:precision_diffN} and~\ref{tab:precision_diffN_relion} that for the dataset of	EMD-5976 and EMD-5977, Algorithm~\ref{alg:rot_maxcut} preforms better than RELION for the tested sizes. On the other hand, in the experiment with EMD-0104 and EMD-0105, although Algorithm~\ref{alg:rot_maxcut} improve significantly with the sample size, RELION outperforms Algorithm~\ref{alg:rot_maxcut}  for all sample sizes. 
\begin{table}
	\begin{center}
		\begin{tabular}{|c|c|c|c|}
			\hline
				\multirow{2}{*}{Volumes} & \multicolumn{3}{c|}{Precision}  \\ \cline{2-4}
			&4000 images & 6000 images & 10000 images\\
			\hline		
			EMD-5976 and EMD-5977 & 0.990 & 0.991 & 0.996 \\
			\hline
			EMD-0104 and EMD-0105  & 0.69 & 0.70 & 0.87 \\

			\hline
		\end{tabular}
		\caption{Results of Algorithm~\ref{alg:rot_maxcut} for a different numbers of projections.}
		\label{tab:precision_diffN}
	\end{center}	
\end{table}

\begin{table}
	\begin{center}
		\begin{tabular}{|c|c|c|c|}
			\hline
			\multirow{2}{*}{Volumes} & \multicolumn{3}{c|}{Precision}  \\ \cline{2-4}
			&4000 images & 6000 images & 10000 images\\
			\hline		
			EMD-5976 and EMD-5977 & 0.95 & 0.6 & 0.96 \\
			\hline
			EMD-0104 and EMD-0105  & 0.63 & 0.93 & 0.99 \\
			
			\hline
		\end{tabular}
		\caption{Results of RELION's Classification for a different numbers of projections.}
		\label{tab:precision_diffN_relion}
	\end{center}	
\end{table}

\section{Conclusion}
\label{sec:conclusion}
We presented an algorithm for approximating the class partition of a heterogeneous cryo-EM image data set. We derived theoretical bounds for the algorithm, applied it on simulated data and compared its performance with RELION. As the proposed algorithm is based on the LUD algorithm~\cite{wang2013orientation}, it is applicable only to molecules without symmetry. However, it can be easily combined with abinitio reconstruction algorithms for molecules with symmetry. Moreover, the LUD algorithm can be completely replaced with other orientation assignment algorithms such as~\cite{Shkolnisky2012}, or RELION's algorithm for orientation assignment. The proposed algorithm can also be supplied with confidence information regarding the score between each pair of images. Such confidence information is available as a byproduct of the algorithms~\cite{Shkolnisky2012,pragier2016graph}, and may further improve the robustness of our algorithm to noise.

We noted above that the algorithm ``favors" balanced partitions, and thus the case of unbalanced classes may result in less accurate output. This behavior is an inherent drawback of the max-cut problem. A possible future research direction to alleviate this problem is replacing the max-cut formulation with a different optimization problem.

\section{Acknowledgments}
We would like to thank Dr. Joakim And\'en for his remarks and corrections on a previous version of the paper, Prof. Niv Buchbinder and Prof. Amir Back for their advice regarding optimizations, and the reviewers for there comments.
This research was supported by the European Research Council
(ERC) under the European Union’s Horizon 2020 research and innovation programme (grant agreement 723991 - CRYOMATH), by Award Number R01GM090200 from the NIGMS, by a Fellowship from Jyv\"{a}skyl\"{a} University and the Clore Foundation.
\begin{appendices}
\section{Proof of Lemma~\ref{lem:maxCut_twoSide}}\label{sec:proof_lem_twosided}
As shown in~\cite{goemans1995improved}, the solution to the max-cut problem is given by the rank-1 matrix $\Sigma$ that minimizes $\operatorname{trace}(W\Sigma)$ such that $\Sigma$ is positive semidefinite, $\Sigma = \sigma\sigma^T$ where $\sigma = (\sigma_1, \ldots, \sigma_N)^{T}$, and $\sigma_i \in \{ +1,-1\} $. The sign of $\sigma_i$ encodes the subset of the cut to which vertex $i$ belongs. Due to the constraint $\Sigma = \sigma\sigma^T$, we have that $\Sigma_{ii} = 1$ for all $i$ (where $\Sigma_{ij}$ are the entries of~$\Sigma$). The Goemans-Williamson algorithm discards the rank-1 constraint as well as the constraint $ \sigma_i \in \{-1,1\}$, while keeping the constraints that~$\Sigma$ is positive semi-definite and that $\Sigma_{ii} = 1$. Thus, the Goemans-Williamson algorithm  minimizes $\operatorname{trace}(W\Sigma)$ for $\Sigma$ positive semidefinite with $\Sigma_{ii} = 1$.

We now show that the solution obtained by the Goemans-Williamson algorithm satisfies $\Sigma_{ij} \in [-1,1]$. Since the Goemans-Williamson optimization problem optimizes over positive semidefinite matrices $\Sigma$, all the $2\times 2$ principal minors of $\Sigma$ are non-negative, that is, $\operatorname{det} \Sigma^{ij} \ge 0$, where
\begin{equation*}
\Sigma^{ij} = \left(
\begin{array}{cc}
\Sigma_{ii} & \Sigma_{ij} \\
\Sigma_{ji} & \Sigma_{jj} \\
\end{array}
\right).
\end{equation*}
In other words, $\Sigma_{ii} \Sigma_{jj} - \Sigma_{ij} \Sigma_{ji} \ge 0$. Since $\Sigma_{ii} = 1$, we get that $1 - \Sigma_{ij}\Sigma_{ji} \geq 0$, or, using the symmetry of~$\Sigma$, $1 - \Sigma_{ij}^2 \geq 0$. Thus, $|\Sigma_{ij}| \leq 1$.

Without loss of generality, assume that the adjacency matrix of the bipartite graph $(V,E)$, denoted by $W$, is a block matrix given by
\begin{equation}\label{eq:block W}
W  = \left(\begin{array}{cc}
0_{N_1 \times N_1} & A_{N_1 \times N_2} \\
A^{T}_{N_2 \times N_1} & 0_{N_2 \times N_2}
\end{array}\right),
\end{equation}
where $A$ is some $N_1 \times N_2$ matrix with non-negative entries.
Denote the elements of $A$ in~\eqref{eq:block W} by $a_{ij}$. Then, 
\begin{equation}\label{eq:traceWSIGMA}
\operatorname{trace}(W\Sigma) = \sum_{i=1}^{N_1+N_2} \sum_{j=1}^{N_1+N_2} w_{ij}\Sigma_{ji} = \sum_{i=1}^{N_1}\sum_{j=1}^{N_2} a_{ij}\Sigma_{N_1+j,i} + \sum_{i=1}^{N_1} \sum_{j=1}^{N_{2}} a_{ji} \Sigma_{j,N_1+i}.
\end{equation}
Since $\Sigma_{ij} \in [-1,+1]$, we have that
\begin{equation}\label{eq:a}
a= -2\sum_{i=1}^{N_{1}}\sum_{j=1}^{N_{2}} a_{ij} \leq \operatorname{trace}(W\Sigma).
\end{equation}
Clearly, any cut $G_{1}$, $G_{2}$ of a graph $(V,E)$ can be encoded as a vector consisting of $+1$ and $-1$, whose $i$'th coordinate equals $1$ if node $i$ is in $ G_{1}$, and equals $-1$ if node $i$ is in $G_{2}$. In the case of the matrix $W$ in~\eqref{eq:block W}, we define $\sigma_{\text{opt}}$ to consist of a block of 1's of length $N_{1}$ followed by a block of $-1$'s of length $N_{2}$. It can be easily verified that the cut encoded by $\sigma_\text{opt}$ corresponds to the matrix $ \Sigma_\text{opt} = \sigma_\text{opt}^{} \sigma_\text{opt}^T $, which achieves the bound $a$ in~\eqref{eq:a}. Thus, $ \Sigma_\text{opt} $ is optimal out of all the positive semidefinite matrices. Since the Goemans-Williamson optimization problem is convex, there are no other minima points, and thus Goemans-Williamson algorithm will return $ \sigma_\text{opt} $, which concludes the proof.

\section{Distributions}\label{sec:distributions}

\begin{lemma}\label{lem:max2n_bounded}
	Let $ X $ be a random variable bounded in $ [0,a] $.
	Let $X^j_i$ be i.i.d. random variables with the same distribution as $ X $.
	Let $ Y_j = \sum_{i=1}^{N_j} X_i^j ,~ 1\leq j \leq N$.
	Then, $\max_{1\leq j\leq N} Y_{i}$ is bounded by $ \max_{1\leq j\leq N}N_j\E(X)+a\sqrt{\frac{\log_2{N}}{2}} \sqrt{N_j}$ with probability that converges to $1$ as $N \to \infty$.
\end{lemma}
\begin{proof}
	Applying Hoeffding's inequality to $ Y_j $ gives
	\begin{equation*}
	P\left(Y_j - \E(Y_j) \geq \sigma(Y_j)t\right) \leq e^{\frac{-2Var(Y_j) t^2}{a^2 N_j}}.
	\end{equation*}
	Since $ Y_j $ is a sum of $ N_j $ i.i.d. random variables, we get
	\begin{equation*}
	P\left(Y_j \geq  \E(Y_j) +  \sigma(Y_j)t\right) \leq e^{\frac{-2Var(X) t^2}{a^2}} = e^{-c t^2},
	\end{equation*}
	were $ c = \frac{2Var(X)}{a^2} $ (note that $ c \leq 1 $ because $ X $ is bounded in $ [0,a] $).
Thus
\begin{align}
P \left (Y_1<\E(Y_1)+t\sigma(Y_1), \ldots, Y_{N}<\E(Y_N)+t\sigma(Y_N) \right ) & = \prod\limits_{j=1}^{N} P(Y_j<\E(Y_j)+t\sigma(Y_j)) = \left(1- e^{-ct^2}\right)^{N}  \nonumber \\
& > 1-N e^{-ct^2} = 1 - e^{\log_2 N\ln 2 -ct^2},\label{eq:berineq}
\end{align}
where the inequality in~\eqref{eq:berineq} follows from Bernoulli's inequality. In particular, for $ t = \sqrt{\log_2{N}/c}$ we have
\begin{equation*}
P(\max_{1\leq j\leq N} Y_j <\max_{1\leq j\leq N}\E(Y_j)+\sqrt{\log_2{N}/c}~ \sigma(Y_j)) \geq 1-e^{(\log_2{N})(\ln2 -1)}.
\end{equation*}
Thus,
\begin{equation*}
\lim_{N\to\infty} P(\max_{1\leq j\leq N} Y_{j} >\max_{1\leq j\leq N}\E(Y_j)+\sqrt{\log_2{N}/c}~\sigma(Y_j)) \to 0.
\end{equation*}
Since $  Y_j = \sum_{i=1}^{N_j} X_i^j $, we have
\begin{equation*}
\E(Y_j) = N_j\E(X), ~~~ \sigma(Y_j) = \sqrt{N_j}\sigma(X),
\end{equation*}
and thus,
\begin{equation*}
\lim_{N\to\infty} P(\max_{1\leq j\leq N} Y_{j} >\max_{1\leq j\leq N}N_j\E(X)+\sqrt{\log_2{N}/c}~ \sqrt{N_j} \sigma(X)) \to 0,
\end{equation*}
or, substituting $ c $,
\begin{equation*}
\lim_{N\to\infty} P(\max_{1\leq j\leq N} Y_{j} >\max_{1\leq j\leq N}N_j\E(X)+a\sqrt{\frac{\log_2{N}}{2}}~ \sqrt{N_j}) \to 0.
\end{equation*}
	
\end{proof}

\end{appendices}
\bibliographystyle{plain}
\bibliography{CryoEMUsingMaxCut.bbl}
\end{document}